\documentclass{article}

\usepackage[final]{corl_2020} 
\usepackage{enumitem}
\usepackage{multicol}
\usepackage{amsthm}
\usepackage{amsmath,amssymb,mathtools}
\usepackage{amsfonts}  
\usepackage[ruled,noend]{algorithm2e}
\usepackage{booktabs}
\usepackage{siunitx}
\usepackage{multirow}
\usepackage{xcolor}
\usepackage{soul}
\usepackage{subcaption}
\usepackage{breqn}
\usepackage{natbib}
\usepackage{bbm}
\usepackage{empheq}
\usepackage{wrapfig, blindtext}

\usepackage{algorithmic}

\newtheorem{definition}{Definition}
\newtheorem{theorem}{Theorem}[section]

\newtheorem{obs}{Observation}[theorem]

\title{$f$-IRL: Inverse Reinforcement Learning \\via State Marginal Matching}

\author{
  Tianwei Ni\thanks{Equal contribution, orders determined by dice rolling.\; $^{\dagger}$Equal advising.}, ~ Harshit Sikchi\footnotemark[1],  ~Yufei Wang\footnotemark[1], ~Tejus Gupta\footnotemark[1], ~Lisa Lee\footnotemark[2], ~Benjamin Eysenbach\footnotemark[2] \\
  Carnegie Mellon University \\
  \texttt{\{tianwein, hsikchi, yufeiw2, tejusg, lslee, beysenba\}@cs.cmu.edu} \\
}

\begin{document}
\setlength\intextsep{0pt}
\setlength{\textfloatsep}{0.15cm}
\addtolength{\parskip}{-0.5mm}
\maketitle
\newcommand{\yf}[1]{\textcolor{purple}{yf: #1}}
\newcommand{\tw}[1]{\textcolor{blue}{TW: #1}}
\newcommand{\hs}[1]{\textcolor{red}{hs: #1}}
\newcommand{\E}[2]{\mathbb{E}_{#1}{\left[#2\right]}}

\newcommand{\x}{\mathbf{x}}
\renewcommand{\s}{\mathbf{s}}
\renewcommand{\a}{\mathbf{a}}
\renewcommand{\o}{\mathcal{O}}
\newcommand{\R}{\mathbb{R}}

\newcommand{\sbr}[1]{\left[#1\right]}
\newcommand*{\argmax}{\mathop{\mathrm{argmax}}}
\newcommand{\defeq}{\mathrel{\mathop:}=}
\newcommand{\question}{\stackrel{?}{=}}
\newcommand{\be}{\begin{equation}}
\newcommand{\ee}{\end{equation}}

\newcommand{\rparam}{\theta}
\newcommand{\mypartial}[2]{\frac{\partial#1}{\partial#2}}
\newcommand{\grad}{\nabla}

\newcommand{\fdiv}[2]{D_f(#1\ ||\ #2)}
\newcommand{\kl}[2]{D_{\mathrm{KL}}(#1\ ||\ #2)}
\newcommand{\supp}[1]{\mathrm{supp}(#1)}

\definecolor{lightgreen}{HTML}{90EE90}
\newcommand{\boxedeq}[2]{\begin{empheq}[box={\fboxsep=6pt\fbox}]{align}\label{#1}#2\end{empheq}}

\begin{abstract}
\setcounter{footnote}{0}
Imitation learning is well-suited for robotic tasks where it is difficult to directly program the behavior or specify a cost for optimal control. 
In this work, we propose a method for learning the reward function (and the corresponding policy) to match the expert state density. 
Our main result is the analytic gradient of any $f$-divergence between the agent and expert state distribution w.r.t. reward parameters. Based on the derived gradient, we present an algorithm, $f$-IRL, that recovers a stationary reward function from the expert density by gradient descent. 
We show that $f$-IRL can learn behaviors from a hand-designed target state density \textit{or} implicitly through expert observations. 
Our method outperforms adversarial imitation learning methods in terms of sample efficiency and the required number of expert trajectories on IRL benchmarks. 
Moreover, we show that the recovered reward function can be used to quickly solve downstream tasks, and empirically demonstrate its utility on hard-to-explore tasks and for behavior transfer across changes in dynamics.\footnote{ Project videos and code link are available at 
\url{https://sites.google.com/view/f-irl/home}.
}

\end{abstract}


\newif\ifrss
\rssfalse

\section{Introduction}

Imitation learning (IL) is a powerful tool to design autonomous behaviors in robotic systems. Although reinforcement learning methods promise to learn such behaviors automatically, they have been most successful in tasks with a clear definition of the reward function. Reward design remains difficult in many robotic tasks such as driving a car \cite{pomerleau1989alvinn}, tying a knot \cite{osa2017online}, and human-robot cooperation \cite{hadfield2016cooperative}. Imitation learning is a popular approach to such tasks, since it is easier for an expert teacher to demonstrate the desired behavior rather than specify the reward \cite{atkeson1997robot, henderson2018deep, hwangbo2019learning}.

Methods in IL frameworks are generally split into behavior cloning (BC) \cite{bain1995framework} and inverse reinforcement learning (IRL) \cite{russell1998learning, ng2000algorithms,abbeel2004apprenticeship}.
BC is typically based on supervised learning to regress expert actions from expert observations without the need for further interaction with the environment, but suffers from the covariate shift problem \cite{ross2010efficient}. 
On the other hand, IRL methods aim to learn the reward function from expert demonstrations, and use it to train the agent policy. Within IRL, adversarial imitation learning (AIL) methods (GAIL~\cite{ho2016generative}, AIRL~\cite{fu2017learning}, $f$-MAX~\cite{ghasemipour2019divergence}, SMM~\cite{lee2019efficient}) train a discriminator to guide the policy to match the expert's state-action distribution.

AIL methods learn a \emph{non-stationary} reward by iteratively training a discriminator and taking a single policy update step using the reward derived from the discriminator. After convergence, the learned AIL reward cannot be used for training a new policy from scratch, and is thus discarded. In contrast, IRL methods such as ours learn a \textbf{stationary reward} such that, if the policy is trained from scratch using the reward function until convergence, then the policy will match the expert behavior. We argue that learning a stationary reward function can be useful for solving downstream tasks and transferring behavior across different dynamics.

Traditionally, IL methods assume access to expert demonstrations and minimize some divergence between policy and expert's trajectory distribution. However, in many cases, it may be easier to directly specify the state distribution of the desired behavior rather than to provide fully-specified demonstrations of the desired behavior~\cite{lee2019efficient}. For example, in a safety-critical application, it may be easier to specify that the expert never visits some unsafe states, instead of tweaking reward to penalize safety violations~\citep{dalal2018safe}. 
Similarly, we can specify a uniform density over the whole state space for exploration tasks, or a Gaussian centered at the goal for goal-reaching tasks. Reverse KL instantiation for $f$-divergence in $f$-IRL allows for unnormalized density specification, which further allows for easier preference encoding.

In this paper, we propose a new method, $f$-IRL, that learns a stationary reward function from the expert density via gradient descent. To do so, we derive an analytic gradient of any arbitrary $f$-divergence between the agent and the expert state distribution w.r.t. reward parameters. We demonstrate that $f$-IRL is especially useful in the limited data regime, exhibiting better sample efficiency than prior work in terms of the number of environment interactions and expert trajectories required to learn the MuJoCo benchmark tasks. 
We also demonstrate that the reward functions recovered by $f$-IRL can accelerate the learning of hard-to-explore tasks with sparse rewards, and these same reward functions can be used to transfer behaviors across changes in dynamics.

\begin{table*}[t]
    \centering
    \footnotesize
    \begin{tabular}{ ccccc } 
    \toprule
    \textbf{IL Method} & \textbf{Space} & \textbf{$f$-Divergence} & \textbf{Recover Reward?} \\
    \midrule
    MaxEntIRL~\cite{ziebart2008maximum}, GCL~\cite{finn2016guided} & $\tau$ & Forward Kullback-Leibler & \checkmark  \\ 
    GAN-GCL~\cite{finn2016connection} & $\tau$ & Forward Kullback-Leibler\textsuperscript{*}   & \checkmark   \\
    AIRL~\cite{fu2017learning}, EAIRL~\cite{qureshi2018adversarial} & $\tau$ & Forward Kullback-Leibler\textsuperscript{*}  & \checkmark \\
    EBIL~\cite{liu2020energy} & $\tau$ & Reverse Kullback-Leibler & \checkmark \\
    GAIL~\cite{ho2016generative} & $s,a$ & Jensen-Shannon & $\times$  \\
    $f$-MAX~\cite{ghasemipour2019divergence} & $s,a$ & $f$-divergence & $\times$  \\
    SMM~\cite{lee2019efficient} & $s$ & Reverse Kullback-Leibler & $\times$   \\
    $f$-IRL (Our method) & $s$ & $f$-divergence & \checkmark \\
    \bottomrule
    \end{tabular}
    \vspace{1mm}
    \caption{IL methods vary in the domain of the expert distribution that they model (``space''), the choice of f-divergence, and whether they recover a stationary reward function.
    \textsuperscript{*}GAN-GCL and AIRL use biased IS weights to approximate FKL (see Appendix~\ref{sec:critique}).}
    \label{tab:works}
\end{table*}
\section{Related Work}
\label{sec:background}
IRL methods~\cite{russell1998learning, ng2000algorithms, abbeel2004apprenticeship} obtain a policy by learning a reward function from sampled trajectories of an expert policy. MaxEntIRL~\citep{ziebart2008maximum} learns a stationary reward by maximizing the likelihood of expert trajectories, i.e., it minimizes forward KL divergence in trajectory space under the maximum entropy RL framework. Similar to MaxEntIRL, Deep MaxEntIRL~\cite{wulfmeier2015maximum} and GCL~\cite{finn2016guided} optimize the forward KL divergence in trajectory space. 
A recent work, EBIL~\cite{liu2020energy}, optimizes the reverse KL divergence in the trajectory space by treating the expert state-action marginal as an energy-based model.
Another recent method, RED~\cite{wang2019random}, uses support estimation on the expert data to extract a fixed reward, instead of trying to minimize a $f$-divergence between the agent and expert distribution.

One branch of IRL methods train a GAN~\cite{goodfellow2014generative} with a special structure in the discriminator to learn the reward. This is first justified by \citet{finn2016connection} to connect GCL~\cite{finn2016guided} with GAN, and several methods ~\citep{finn2016connection,fu2017learning,qureshi2018adversarial} follow this direction. 
Our analysis in Appendix~\ref{sec:critique} suggests that the importance-sampling weights used in these prior methods may be biased. We show that AIRL does not minimize the reverse RL in state-marginal space (as argued by~\citep{ghasemipour2019divergence}). Moreover, AIRL~\cite{fu2017learning} uses expert state-action-next state transitions, while our method can work in a setting where only expert states are provided.

A set of IL methods~\cite{ho2016generative,ghasemipour2019divergence} 
use a discriminator to address the issue of running RL in the \textit{inner} loop as classical IRL methods. Instead, these methods directly optimize the policy in the \textit{outer} loop using adversarial training. These methods can be shown to optimize the Jensen-Shannon, and a general $f$-divergence respectively, but do not learn a reward function. SMM \cite{lee2019efficient} optimizes the reverse KL divergence between the expert and policy state marginals but also does not recover a reward function due to its fictitious play approach. SQIL~\cite{reddy2019sqil} and DRIL~\cite{brantley2019disagreement} utilize regularized behavior cloning for imitation without recovering a reward function.
Unlike these prior methods, $f$-IRL can optimize any $f$-divergence between the state-marginal of the expert and the agent, while also recovering a stationary reward function. Table~\ref{tab:works} summarizes the comparison among imitation learning methods.

\section{Preliminaries}
\label{sec:prelim}

In this section, we review notation on maximum entropy (MaxEnt) RL~\cite{levine2018reinforcement} and state marginal matching (SMM)~\cite{lee2019efficient} that we build upon in this work.

\textbf{MaxEnt RL}.\; Consider a Markov Decision Process (MDP) represented as a tuple $(\mathcal{S}, \mathcal{A}, \mathcal{P}, r, \rho_0, T)$ with state-space $\mathcal{S}$, action-space $\mathcal{A}$, dynamics $\mathcal{P}:\mathcal{S} \times \mathcal{A} \times \mathcal{S} \rightarrow [0,1]$, reward function $r(s,a)$, initial state distribution $\rho_0$, and horizon $T$.
The optimal policy $\pi$ under the maximum entropy framework~\cite{ziebart2010modeling} maximizes the objective $\sum_{t=1}^T \E{\rho_{\pi,t}(s_t,a_t)}{r(s_t,a_t) + \alpha H(\cdot|s_t)}$. Here $\rho_{\pi,t}$ is the state-action marginal distribution of policy $\pi$ at timestamp $t$, and $\alpha > 0$ is the entropy temperature.

Let $r_\rparam(s)$ be a parameterized differentiable reward function only dependent on state. Let trajectory $\tau$ be a time series of visited states $\tau=(s_0, s_1,\dots,s_T)$. The optimal MaxEnt trajectory distribution $\rho_\rparam(\tau)$ under reward $r_\rparam$ can be computed as $\rho_\rparam(\tau) = \frac1Z p(\tau)e^{r_\rparam(\tau)/\alpha}$, where
$$p(\tau)=\rho_0(s_0)\prod_{t=0}^{T-1} p(s_{t+1}|s_t, a_t)\;,\quad r_{\rparam}(\tau)=\sum_{t=1}^T r_{\rparam}(s_t),\quad Z = \int p(\tau) e^{r_{\rparam}(\tau)/\alpha}d\tau.
$$

Slightly overloading the notation, the optimal MaxEnt state marginal distribution $\rho_\rparam(s)$ under reward $r_\rparam$ is obtained by marginalization:
\begin{equation}
        \rho_\rparam(s) \propto \int p(\tau) e^{r_\rparam(\tau)/\alpha} \eta_{\tau}(s) d\tau
\end{equation}
where $\eta_{\tau}(s) \triangleq  \sum_{t=1}^T \mathbbm{1}(s_t = s)$ is the visitation count of a state $s$ in a particular trajectory $\tau$. 

\textbf{State Marginal Matching}.\; Given the expert state density $p_E(s)$, one can train a policy to match the expert behavior by minimizing the following $f$-divergence objective:
\begin{equation}
\label{eq:fdiv_ori}
    L_f(\rparam) = \fdiv{\rho_E(s)}{\rho_\rparam(s)}
\end{equation}
where common choices for the $f$-divergence $D_f$~\cite{ali1966general,ghasemipour2019divergence} include forward KL divergence, reverse KL divergence, and Jensen-Shannon divergence. Our proposed $f$-IRL algorithm will compute the analytical gradient of Eq.~\ref{eq:fdiv_ori} w.r.t. $\theta$ and use it to optimize the reward function via gradient descent.

\section{Learning Stationary Rewards via State-Marginal Matching}
\label{sec:method}
In this section, we describe our algorithm $f$-IRL, which takes the expert state density as input, and optimizes the $f$-divergence objective (Eq.~\ref{eq:fdiv_ori}) via gradient descent. Our algorithm trains a policy whose state marginal is close to that of the expert, and a corresponding stationary reward function that would produce the same policy if the policy were trained with MaxEnt RL from scratch.

\subsection{Analytic Gradient for State Marginal Matching in $f$-divergence}

One of our main contributions is the exact gradient of the $f$-divergence objective (Eq. \ref{eq:fdiv_ori}) w.r.t. the reward parameters $\theta$. This gradient will be used by $f$-IRL to optimize Eq.~\ref{eq:fdiv_ori} via gradient descent. The proof is provided in Appendix \ref{app:deriv}.

\begin{theorem}[\textbf{$f$-divergence analytic gradient}] The analytic gradient of the $f$-divergence $L_f(\theta)$ between state marginals of the expert ($\rho_E$) and the soft-optimal agent w.r.t. the reward parameters $\theta$ is given by:
\label{thm:fdiv}
\begin{equation}
\begin{split}
    \label{eq:fdiv}
   \nabla_\theta L_f(\theta)= \frac{1}{\alpha T}\mathrm{cov}_{\tau\sim \rho_{\theta}(\tau)}\left(\sum_{t=1}^{T} h_f\left(\frac{\rho_E(s_t)}{\rho_{\theta}(s_t)}\right),\sum_{t=1}^{T} \nabla_\theta r_\theta(s_t)\right)
\end{split}
\end{equation}
where $h_f(u) \triangleq f(u) - f'(u)u$, $\rho_E(s)$ is the expert state marginal and $\rho_{\theta}(s)$ is the state marginal of the soft-optimal agent under the reward function $r_{\theta}$, and the covariance is taken under the agent's trajectory distribution $\rho_{\theta}(\tau)$.\footnote{Here we assume $f$ is differentiable, which is often the case for common $f$-divergence (e.g. KL divergence).}
\end{theorem}

\begin{table}[t]
    \vspace{1mm}
    \centering
    \begin{tabular}{cccc}
        \toprule
       Name  & $f$-divergence $\fdiv{P}{Q}$ & Generator $f(u)$ & $h_f(u)$ \\
       \midrule
        \textbf{FKL} & $\int p(x) \log \frac{p(x)}{q(x)}dx$ &  $u\log u$ & $-u$   \\
        \textbf{RKL} & $\int q(x) \log \frac{q(x)}{p(x)}dx$ &  $-\log u$ & $1-\log u$   \\
        \textbf{JS}  & \scalebox{0.75}{$\frac{1}{2}\int p(x) \log \frac{2p(x)}{p(x)+q(x)} + q(x) \log \frac{2q(x)}{p(x)+q(x)}dx$} &  $u\log u -(1+u)\log \frac {1+u}2$ & $-\log(1+u)$   \\
      \bottomrule
    \end{tabular}
    \vspace{1mm}
    \caption{
    {Selected list of $f$-divergences $\fdiv{P}{Q}$ with generator functions $f$ and $h_f$ defined in Theorem \ref{thm:fdiv}, where $f$ is convex, lower-semicontinuous and $f(1)=0$.}
    }
    \label{tab:f-div}
    \vspace{-1.2em}
\end{table}

Choosing the $f$-divergence to be Forward Kullback-Leibler (FKL), Reverse Kullback-Leibler (RKL), or Jensen-Shannon (JS) instantiates $h_f$ (see Table~\ref{tab:f-div}).  Note that the gradient of the RKL objective has a special property in that we can specify the expert as an \textit{unnormalized} log-density (i.e. energy), since in $h_{\mathrm{RKL}}(\frac{\rho_E(s)}{\rho_{\theta}(s)})=1- \log \rho_E(s) + \log \rho_{\theta}(s)$, the normalizing factor of $\rho_E(s)$ does not change the gradient (by linearity of covariance). This makes density specification much easier in a number of scenarios.
Intuitively, since $h_f$ is a monotonically decreasing function ($h'_f(u) = -f''(u)u<0$) over $\mathbb{R}^+$, the gradient descent tells the reward function to increase the rewards of those state trajectories that have higher sum of density ratios $\sum_{t=1}^T \frac{\rho_E(s_t)}{\rho_{\theta}(s_t)}$ so as to minimize the objective.

\subsection{Learning a Stationary Reward by Gradient Descent}

We now build upon Theorem~\ref{thm:fdiv} to design a practical algorithm for learning the reward function $r_\rparam$ (Algorithm.~\ref{algo:full}). Given expert information (state density or observation samples) and an arbitrary $f$-divergence, the algorithm alternates between using MaxEnt RL with the current reward, and updating the reward parameter using gradient descent based on the analytic gradient.

If the provided expert data is in the form of expert state density $\rho_E(s)$, we can fit a density model $\hat\rho_{\theta}(s)$ to estimate agent state density $\rho_{\theta}(s)$ and thus estimate the density ratio required in gradient. If we are given samples from expert observations $s_E$, we can fit a discriminator $D_\omega(s)$ in each iteration to estimate the density ratio by optimizing the binary cross-entropy loss:
\begin{equation}
\label{eq:disc}
\max_\omega \E{s\sim s_E}{\log D_\omega(s)} + \E{s\sim \rho_\theta(s)}{\log (1- D_\omega(s)}
\end{equation}
where the optimal discriminator satisfies $D^*_\omega(s) = \frac{\rho_E(s)}{\rho_E(s) +\rho_{\theta}(s)}$~\cite{goodfellow2014generative}, thus the density ratio can be estimated by $\frac{\rho_E(s)}{\rho_{\theta}(s)}\approx \frac{D_\omega(s)}{1-D_\omega(s)}$, which is the input to $h_f$.

\begin{algorithm}
\SetAlgoLined
 \SetKwInOut{Input}{Input}
 \SetKwInOut{Output}{Output}
 \Input{Expert state density $\rho_E(s)$ or expert observations $s_E$ , $f$-divergence}

 \Output{Learned reward $r_\rparam$, Policy $\pi_\theta$}
  Initialize $r_\rparam$, and density estimation model (provided $\rho_E(s)$) or disciminator $D_\omega$ (provided $s_E$)
  
 \For{$i \leftarrow 1$ \KwTo $Iter$}{

 $\pi_\theta \leftarrow$ MaxEntRL($r_\rparam$) and collect agent trajectories $\tau_\theta$
 
\If{provided $\rho_E(s)$}
  {Fit the density model $\hat\rho_{\theta}(s)$ to the state samples from $\tau_\theta$}
\Else
  {\tcp{provided $s_E$ }Fit the discriminator $D_\omega$ by Eq. \ref{eq:disc} using expert and agent state samples from $s_E$ and $\tau_\theta$}
  
  Compute sample gradient $\hat \grad_\rparam L_f(\rparam)$ for Eq. \ref{eq:fdiv} over $\tau_\theta$
  
  $\rparam \leftarrow \rparam - \lambda \hat\grad_\rparam L_f(\rparam)$
 }
 \caption{Inverse RL via State Marginal Matching ($f$-IRL)}
 \label{algo:full}
\end{algorithm}

\subsection{Robust Reward Recovery under State-only Ground-truth Reward}
\label{sec:robust}

IRL methods are different from IL methods in that they recover a reward function in addition to the policy. A hurdle in this process is often the reward ambiguity problem, explored in~\cite{ng1999policy,fu2017learning}. This ambiguity arises due to the fact that the optimal policy remains unchanged under the following reward transformation~\cite{ng1999policy}:
\begin{equation}
    \hat{r}(s,a,s') = r_{\mathrm{gt}}(s,a,s') + \gamma\Phi(s') - \Phi(s)
\end{equation}
for any function $\Phi$. In the case where the ground-truth reward is a function over states only (i.e., $r_{\mathrm{gt}}(s)$), $f$-IRL is able to recover the \textit{disentangled} reward function ($r_{\text{IRL}}$) that matches the ground truth reward  $r_{\mathrm{gt}}$ up to a constant. The obtained reward function is robust to different dynamics -- for any underlying dynamics, $r_{\text{IRL}}$ will produce the same optimal policy as $r_{\mathrm{gt}}$. We formalize this claim in Appendix~\ref{sec:proof_robust} (based on Theorem 5.1 of AIRL~\cite{fu2017learning}).

AIRL uses a special parameterization of the discriminator to learn state-only rewards. A disadvantage of their approach is that AIRL needs to approximate a separate reward-shaping network apart from the reward network. In contrast, our method naturally recovers a state-only reward function.

\subsection{Practical Modification in the Exact Gradient}
\label{sec:trick}

In practice with high-dimensional observations, when the agent's current trajectory distribution is far off from the expert trajectory distribution, we find that there is little supervision available through our derived gradient, leading to slow learning. Therefore, when expert trajectories are provided, we bias the gradient (Eq. \ref{eq:fdiv}) using a mixture of agent and expert trajectories inspired by GCL~\cite{finn2016guided}, which allows for richer supervision and faster convergence. Note that at convergence, the gradient becomes unbiased as the agent's and expert's trajectory distribution matches.

\begin{equation}
\begin{split}
    \label{eq:fdiv_biased}
   \Tilde{\nabla}_\theta L_f(\theta):= \frac{1}{\alpha T}\mathrm{cov}_{\tau\sim \frac{1}{2}(\rho_{\theta}(\tau) + \rho_E(\tau))}\left(\sum_{t=1}^{T} h_f\left(\frac{\rho_E(s_t)}{\rho_{\theta}(s_t)}\right),\sum_{t=1}^{T} \nabla_\theta r_\theta(s_t)\right)
\end{split}
\end{equation}
where the expert trajectory distribution $\rho_E(\tau)$ is uniform over samples $\tau_E$.

\section{Experiments}
\label{sec:experiments}

In our experiments, we seek answers to the following questions:
\begin{enumerate}
    \item Can $f$-IRL learn a policy that matches the given expert state density?
    \item Can $f$-IRL learn good policies on high-dimensional continuous control tasks in a sample-efficient manner? 
    \item Can $f$-IRL learn a reward function that induces the expert policy?
    \item How can learning a stationary reward function help solve downstream tasks?
\end{enumerate}

\textbf{Comparisons}.\; To answer these questions, we compare $f$-IRL against two classes of existing imitation learning algorithms: (1) those that learn only the policy, including Behavior Cloning (BC), GAIL~\cite{ho2016generative}, and $f$-MAX-RKL\footnote{A variant of AIRL~\cite{fu2017learning} proposed in~\cite{ghasemipour2019divergence}  only learns a policy and does not learn a reward.}~\cite{ghasemipour2019divergence}; and (2) IRL methods that learn both a reward and a policy simultaneously, including MaxEnt IRL~\cite{ziebart2008maximum} and AIRL \cite{fu2017learning}. The rewards/discriminators of the baselines are parameterized to be state-only.  We use SAC~\cite{haarnoja2018soft} as the base MaxEnt RL algorithm. Since the original AIRL uses TRPO~\cite{schulman2015trust}, we re-implement a version of AIRL that uses SAC as the underlying RL algorithm for fair comparison. For our method ($f$-IRL), MaxEnt IRL, and AIRL, we use a MLP for reward parameterization.

\textbf{Tasks}.\; We evaluate the algorithms on several tasks:
\begin{itemize}
    \item \textbf{Matching Expert State Density}: In Section~\ref{sec:matching-expert-density}, the task is to learn a policy that matches the given expert state density. 
    \item \textbf{Inverse Reinforcement Learning Benchmarks}: In Section~\ref{sec:irl-experiments}, the task is to learn a reward function and a policy from expert trajectory samples. We collected expert trajectories by training SAC~\cite{haarnoja2018soft} to convergence on each environment. We trained all the methods using varying numbers of expert trajectories $\{1, 4, 16\}$ to test the robustness of each method to the amount of available expert data. 
    \item \textbf{Using the Learned Reward for Downstream Tasks}: In Section~\ref{sec:exp-downstream-tasks}, we first train each algorithm to convergence, then use the learned reward function to train a new policy on a related downstream task. We measure the performance on downstream tasks for evaluation.
\end{itemize} 

We use five MuJoCo continuous control locomotion environments~\cite{todorov2012mujoco, brockman2016openai} with joint torque actions, illustrated in Figure~\ref{fig:envs}.
Further details about the environment, expert information (samples or density specification), and hyperparameter choices can be found in Appendix \ref{appendix:implementation}. 

\begin{figure}[!h]
\vspace{2mm}
    \centering
    \begin{subfigure}[t]{0.19\linewidth}
        \centering
        \includegraphics[width=\linewidth]{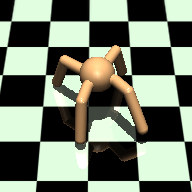}
    \end{subfigure}%
    ~ 
    \begin{subfigure}[t]{0.19\linewidth}
        \centering
        \includegraphics[width=\linewidth]{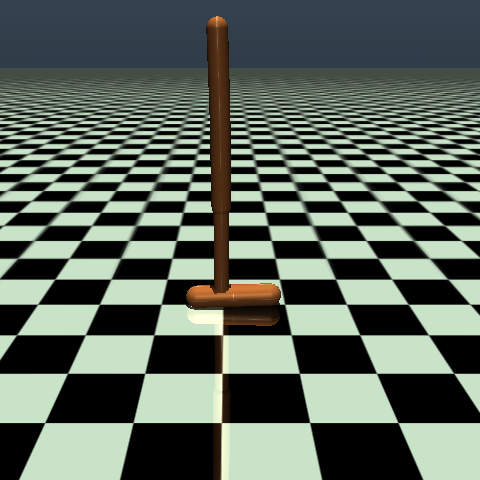}
    \end{subfigure}%
    ~ 
        \begin{subfigure}[t]{0.19\linewidth}
        \centering
        \includegraphics[width=\linewidth]{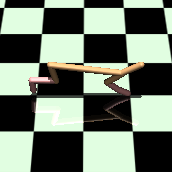}
    \end{subfigure}%
    ~ 
    \begin{subfigure}[t]{0.19\linewidth}
        \centering
        \includegraphics[width=\linewidth]{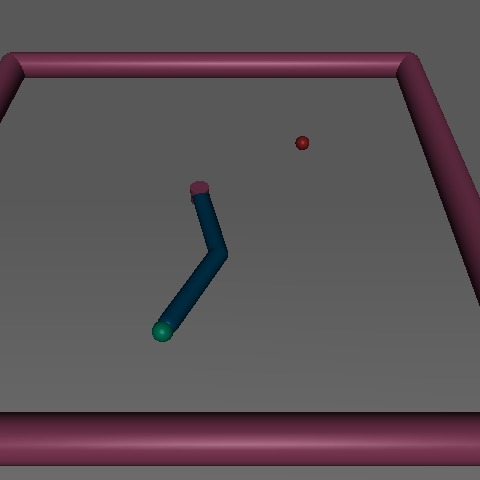}
    \end{subfigure}%
    ~ 
    \begin{subfigure}[t]{0.19\linewidth}
        \centering
        \includegraphics[width=\linewidth]{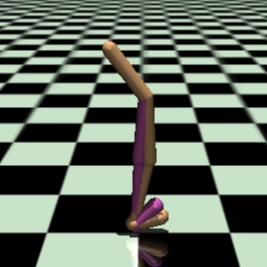}
    \end{subfigure}
    \caption{ \textbf{Environments}: (left to right) Ant-v2, Hopper-v2, HalfCheetah-v2, Reacher-v2, and Walker2d-v2.}
    \label{fig:envs}
\end{figure}

\subsection{Matching the Specified Expert State Density}
\label{sec:matching-expert-density}
First, we check whether $f$-IRL can learn a policy that matches the given expert state density of the fingertip of the robotic arm in the 2-DOF Reacher environment. We evaluate the algorithms using two different expert state marginals: (1) a Gaussian distribution centered at the goal for single goal-reaching, and (2) a mixture of two Gaussians, each centered at one goal. 
Since this problem setting assumes access to the expert density only, we use importance sampling to generate expert samples required by the baselines.

In Figure~\ref{fig:reacher-kl}, we report the estimated forward and reverse KL divergences in state marginals between the expert and the learned policy. 
For $f$-IRL and MaxEnt IRL, we use Kernel Density Estimation (KDE) to estimate the agent's state marginal. We observe that the baselines demonstrate unstable convergence, which might be because those methods optimize the $f$-divergence approximately. Our method \{FKL, JS\} $f$-IRL outperforms the baselines in the forward KL and the reverse KL metric, respectively.

\begin{figure}
    \centering
    \begin{tabular}{cc}
    \multicolumn{2}{c}{\includegraphics[width=0.95\textwidth]{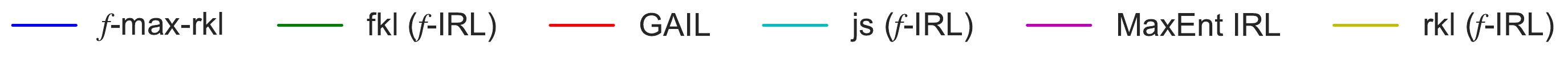}}\\ 
    \includegraphics[width=0.482\textwidth]{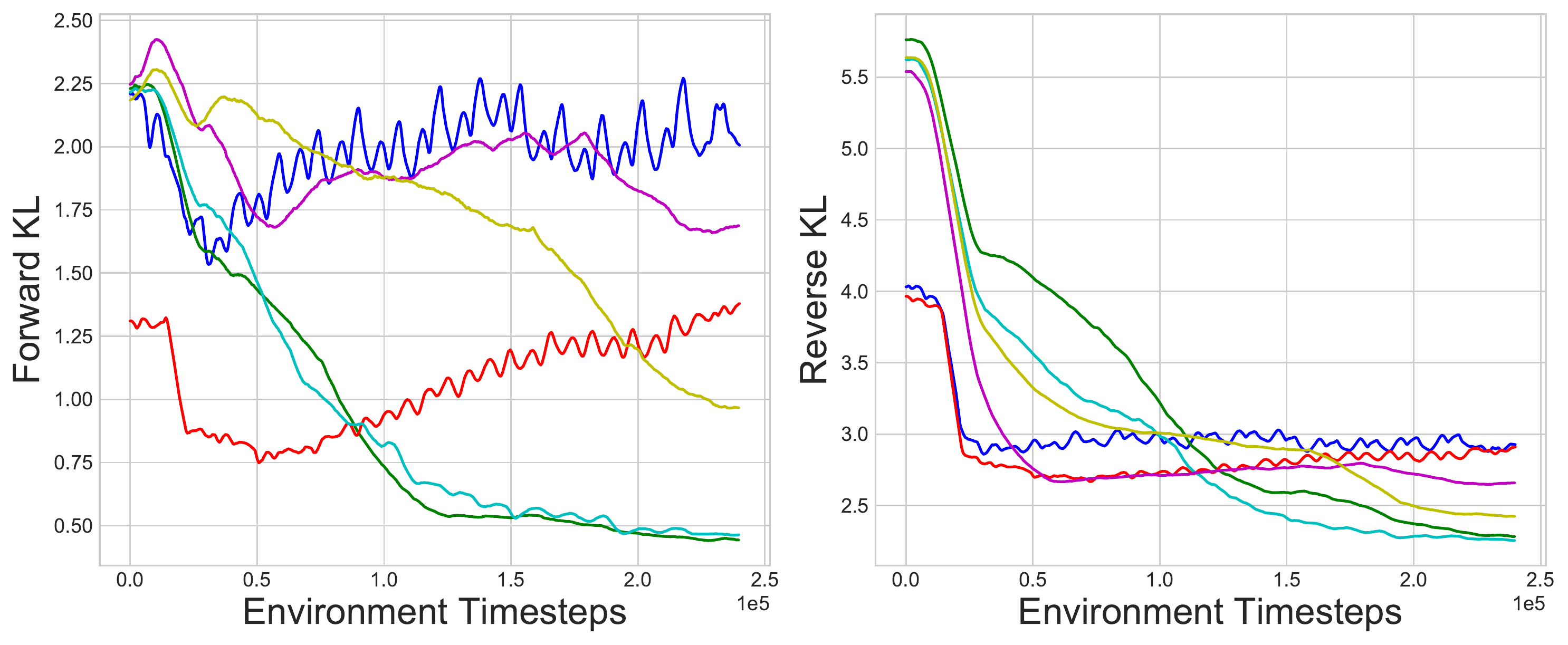}
    & \includegraphics[width=0.482\textwidth]{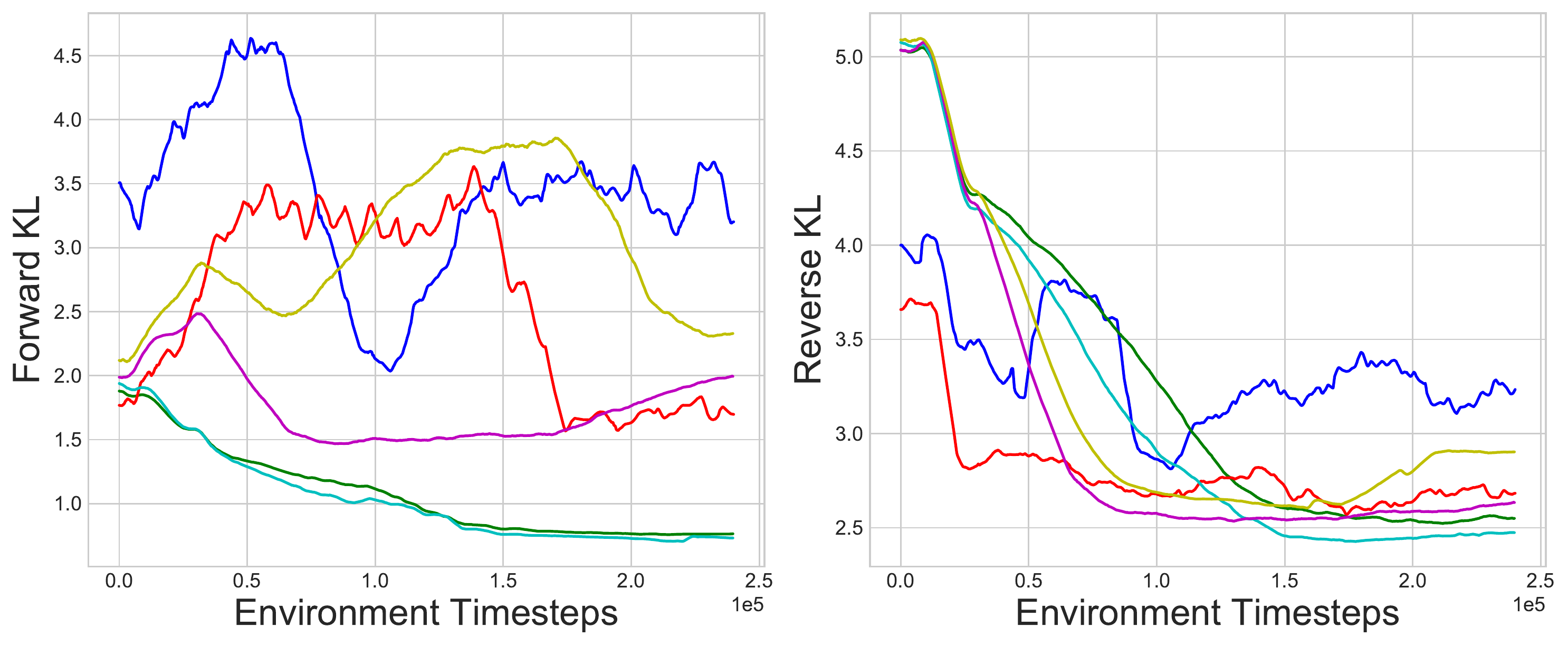}\\
    (a) Expert Density: Gaussian & 
    (b) Expert Density: Mixture of two Gaussians\\
    \end{tabular}
 \caption{Forward (left) and Reverse (right) KL curves in the Reacher environment for different expert densities of all methods. Curves are smoothed in a window of 120 evaluations.}
    \label{fig:reacher-kl}
\end{figure}

\subsection{Inverse Reinforcement Learning Benchmarks}\label{sec:irl-experiments}

Next, we compare $f$-IRL and the baselines on IRL benchmarks, where the task is to learn a reward function and a policy from expert trajectory samples. We use the modification proposed in Section~\ref{sec:trick} to alleviate the difficulty in optimizing the $f$-IRL objective with high-dimensional states. 

\textbf{Policy Performance}.\; We check whether $f$-IRL can learn good policies on high-dimensional continuous control tasks in a sample-efficient manner from expert trajectories.
Figure~\ref{fig:mujoco-curves} shows the learning curves of each method in the four environments with \textit{one} expert trajectory provided. $f$-IRL and MaxEnt IRL demonstrate much faster convergence in most of the tasks than $f$-MAX-RKL. Table~\ref{tab:mujoco_all} shows the final performance of each method in the four tasks, measured by the ratio of agent returns (evaluated using the ground-truth reward) to expert returns.\footnote{The unnormalized agent and expert returns are reported in Appendix~\ref{app:result}.} While MaxEnt IRL provides a strong baseline, $f$-IRL outperforms all baselines on most tasks especially in Ant, where the FKL ($f$-IRL) has much higher final performance and is less sensitive to the number of expert trajectories compared to the baselines. 
In contrast, we found the original implementation of $f$-MAX-RKL to be extremely sensitive to hyperparameter settings. We also found that AIRL performs poorly even after tremendous tuning, similar to the findings in~\cite{liu2019state,liu2020energy}.

\textbf{Recovering the Stationary Reward Function}.\;
We also evaluate whether $f$-IRL can recover a stationary reward function that induces the expert policy.
To do so, we train a SAC agent from scratch to convergence using the reward model obtained from each IRL method. We then evaluate the trained agents using the ground-truth reward to test whether the learned reward functions are good at inducing the expert policies. 

Table~\ref{tab:mujoco_final} shows the ratio of the final returns of policy  trained from scratch using the rewards learned from different IRL methods with one expert trajectory provided, to expert returns.  Our results show that MaxEnt IRL and $f$-IRL are able to learn \textit{stationary} rewards that can induce a policy close to the optimal expert policy. 

\begin{figure}[t]
\vspace{1mm}
\centering
\includegraphics[width=0.95\columnwidth]{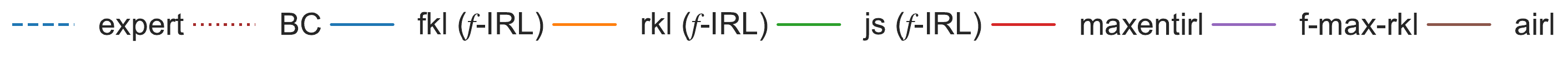}\\
    \includegraphics[width=1.0\columnwidth,height=3.2cm]{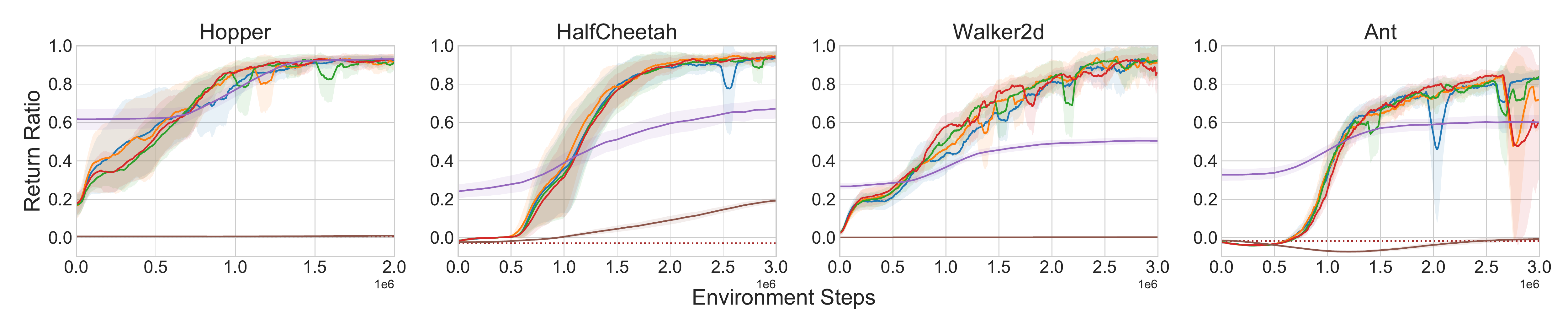}
 \caption{Training curves for $f$-IRL and 4 other baselines - BC, MaxEnt IRL, $f$-MAX-RKL and AIRL with one expert demonstration. Solid curves depict the mean of 3 trials and the shaded area shows the standard deviation. The dashed blue line represents the expert performance and the dashed red line shows the performance of a BC agent at convergence.}
 \vspace{1mm}
\label{fig:mujoco-curves}
\end{figure}

\begin{table}
    \vspace{2mm}
    \centering
    \footnotesize
    \begin{tabular}{c|ccc|ccc|ccc|ccc}
        \toprule
         Method & \multicolumn{3}{c|}{Hopper} & \multicolumn{3}{c|}{Walker2d} & \multicolumn{3}{c|}{HalfCheetah} & \multicolumn{3}{c}{Ant}\\
        \midrule
        Expert return &
  \multicolumn{3}{c|}{3592.63 $\pm$ 19.21} & \multicolumn{3}{c|}{5344.21 $\pm$ 84.45} & \multicolumn{3}{c|}{12427.49 $\pm$ 486.38} & \multicolumn{3}{c}{5926.18 $\pm$ 124.56}\\
         \# Expert traj
           & 1 & 4 & 16
           & 1 & 4 & 16
           & 1 & 4 & 16
           & 1 & 4 & 16 \\
       \midrule

        BC 
           & 0.00 & 0.13 & 0.16
           & 0.00 & 0.05 & 0.08
           & 0.00 & 0.01 & 0.02
           & 0.00 & 0.22 & 0.47 \\
        MaxEnt IRL  
           & 0.93 & 0.92 & \textbf{0.94}
           & 0.88 & 0.88 & \textbf{0.91} 
           & \textbf{0.95} & \textbf{0.98} & 0.91 
           & 0.54 & 0.71 & 0.81 \\
        $f$-MAX-RKL  
           & \textbf{0.94} & \textbf{0.93} & 0.91
           & 0.49 & 0.49 & 0.47
           & 0.71 & 0.41 & 0.65 
           & 0.60 & 0.65 & 0.62 \\
        AIRL  & 0.01 & 0.01 & 0.01
           & 0.00 & 0.00 & 0.00
           & 0.19 & 0.19 & 0.19
           & 0.00 & 0.00 & 0.00 \\
           \midrule
        FKL ($f$-IRL)  
           & 0.93 & 0.90 & 0.93 
           & \textbf{0.90} & 0.90 & 0.90
           & 0.94 & 0.97 & 0.94 
           & \textbf{0.82} & \textbf{0.83} & \textbf{0.84} \\
        RKL ($f$-IRL)  
           & 0.93 & 0.92 & 0.93
           & 0.89 & 0.90 & 0.85 
           & \textbf{0.95} & 0.97 & \textbf{0.96} 
           & 0.63 & 0.82 & 0.81 \\
        JS ($f$-IRL)  
           & 0.92 & \textbf{0.93} & \textbf{0.94}
           & 0.89 & \textbf{0.92} & 0.88
           & 0.93 & \textbf{0.98} & 0.94 
           & 0.77 & 0.81 & 0.73 \\
        \bottomrule
    \end{tabular}
    \vspace{1mm}
    \caption{ We report the ratio between the average return of the trained (stochastic) policy vs. that of the expert policy for different IRL algorithms using 1, 4 and 16 expert trajectories. All results are averaged across 3 seeds. Negative ratios are clipped to zero.}
    \label{tab:mujoco_all}
\end{table}

\begin{table}[!h]
    \vspace{2mm}
    \centering
    \begin{tabular}{c|cccc}
        \toprule
        Method & Hopper & Walker2d & HalfCheetah & Ant\\
        \midrule
        AIRL & - & - & -0.03 & - \\
        MaxEntIRL & \textbf{0.93} & \textbf{0.92} & 0.96 & 0.79 \\
        $f$-IRL  & \textbf{0.93}  & 0.88 & \textbf{1.02}  & \textbf{0.82}\\
        \bottomrule
    \end{tabular}
    \vspace{1mm}
    \caption{ The ratios of final return of the obtained policy against expert return across IRL methods. We average $f$-IRL over FKL, RKL, and JS. `-' indicates that we do not test learned rewards since AIRL does poorly at these tasks in Table~\ref{tab:mujoco_all}.}
    \label{tab:mujoco_final}
\end{table}

\begin{table}[!h]
    \centering
    \vspace{2mm}
    \begin{tabular}{cccc|c}
        \toprule
        Policy Transfer & AIRL &
        MaxEntIRL &$f$-IRL & Ground-truth\\
        using GAIL &  &
         & & Reward\\
        \midrule
        -29.9 & 130.3 & \textbf{145.5}& 141.1 & 315.5\\ 
        \bottomrule
    \end{tabular}
    \vspace{1mm}
    \caption{ Returns obtained after transferring the policy/reward on modified Ant environment using different IL methods.\label{tab:policy-transfer}}
\end{table}

\begin{figure}[!h]
\vspace{2mm}
    \centering
    \includegraphics[width=0.495\textwidth]{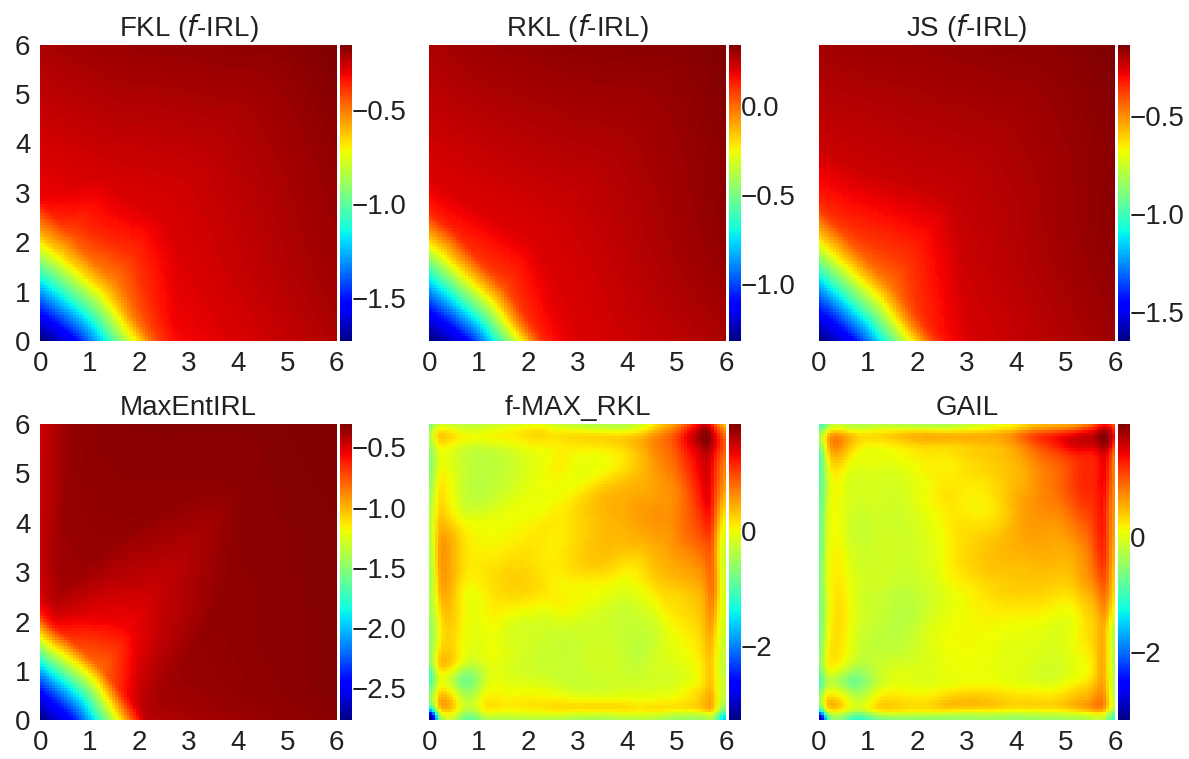}
    \includegraphics[width=0.495\textwidth]{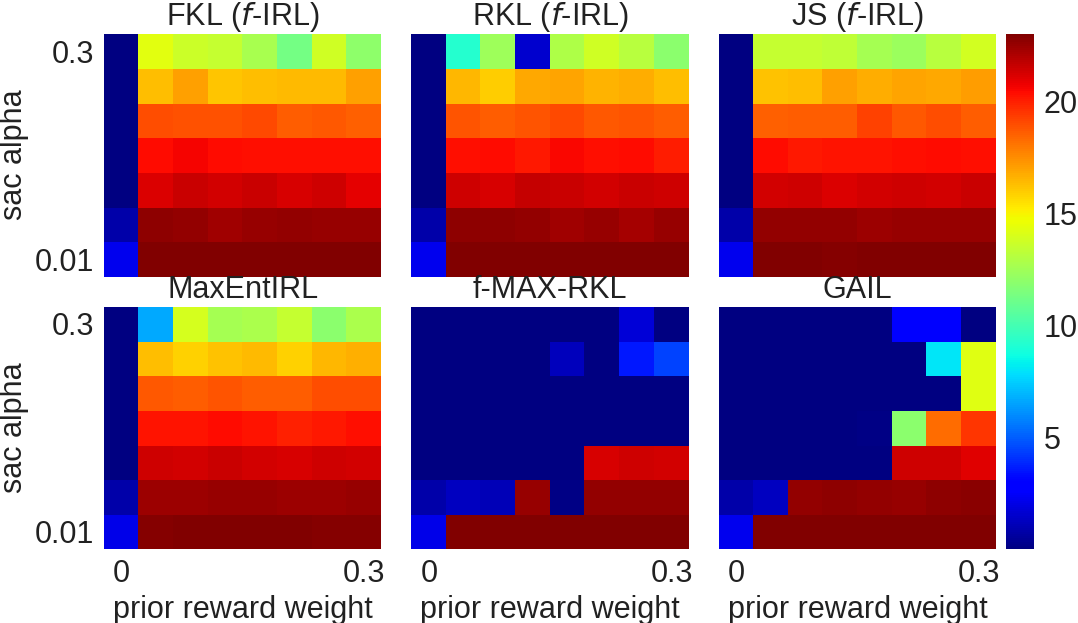}
    
    \caption{ Left: Extracted final reward of all compared methods for the uniform expert density in the point environment. Right: The task return (in terms of $r_{\text{task}}$)  with different $\alpha$ and prior reward weight $\lambda$. The performance of vanilla SAC is shown in the leftmost column with $\lambda = 0$ in each subplot.}
    \label{fig:uniform_utility}
\end{figure}

\subsection{Using the Learned Stationary Reward for Downstream Tasks}\label{sec:exp-downstream-tasks}

Finally, we investigate how the learned stationary reward can be used to learn related, downstream tasks.

\textbf{Reward prior for downstream hard-exploration tasks.} We first demonstrate the utility of the learned stationary reward by using it as a prior reward for the downstream task. Specifically, we construct a didactic point mass environment that operates under linear dynamics in a 2D $6 \times 6$ room, and actions are restricted to $[-1,1]$. The prior reward is obtained from a \textit{uniform} expert density over the whole state space, and is used to ease the learning in the hard-exploration task, where we design a difficult goal to reach with distraction rewards (full details in appendix \ref{appendix:implementation}). 

We use the learned prior reward $r_{\text{prior}}$ to augment the task reward $r_{\text{task}}$ as follows: 
\begin{equation}
r(s,a,s') = r_{\text{task}}(s,a,s') + \lambda (\gamma r_{\text{prior}}(s') - r_{\text{prior}}(s)) 
\end{equation}
where $\lambda \ge 0$ is the weight of prior reward and $\gamma$ is the discount factor.
The main theoretical result of~\cite{ng1999policy} dictates that adding a potential-based reward in this form will not change the optimal policy. 
GAIL and $f$-MAX-RKL do not extract a reward function but rather a discriminator, so we derive a prior reward from the discriminator in the same way as~\cite{ghasemipour2019divergence, ho2016generative}. 

Figure \ref{fig:uniform_utility} illustrates that the reward recovered by \{FKL, RKL, JS\} $f$-IRL and the baseline MaxEnt IRL are similar: the reward increases as the distance to the agent's start position, the bottom left corner, increases. This is intuitive for achieving the target uniform density: states farther away should have higher rewards. $f$-MAX-RKL and GAIL's discriminator demonstrate a different pattern which does not induce a uniform state distribution. 
The leftmost column in the Figure \ref{fig:uniform_utility} (Right) shows the poor performance of SAC training without reward augmentation ($\lambda=0$). This verifies the difficulty in exploration for solving the task. 
We vary $\lambda$ in the x-axis, and $\alpha$ in SAC in the y-axis, and plot the final task return (in terms of $r_{\text{task}}$) as a heatmap in the figure. The presence of larger red region in the heatmap shows that our method can extract a prior reward that is more robust and effective in helping the downstream task attain better final performance with its original reward.

\textbf{Reward transfer across changing dynamics.}
Lastly, we evaluate the algorithms on transfer learning across different environment dynamics, following the setup from~\cite{fu2017learning}. In this setup, IL algorithms are provided expert trajectories from a quadrupedal ant agent which runs forward. The algorithms are tested on an ant with two of its legs being disabled and shrunk. This requires the ant to significantly change its gait to adapt to the disabled legs for running forward. 

We found that a forward-running policy obtained by GAIL fails to transfer to the disabled ant. In contrast, IRL algorithms such as $f$-IRL are successfully able to learn the expert's reward function using expert demonstrations from the quadrupedal ant, and use the reward to train a policy on the disabled ant. The results in Table~\ref{tab:policy-transfer} show that the reward learned by $f$-IRL is robust and enables the agent to learn to move forward with just the remaining two legs.

\section{Conclusion}
In summary, we have proposed $f$-IRL, a practical IRL algorithm that distills an expert's state distribution into a stationary reward function. Our $f$-IRL algorithm can learn from either expert samples (as in traditional IRL), or a specified expert density (as in SMM~\cite{lee2019efficient}), which opens the door to supervising IRL with different types of data. These types of supervision can assist agents in solving tasks faster, encode preferences for how tasks are performed, and indicate which states are unsafe and should be avoided. Our experiments demonstrate that $f$-IRL is more sample efficient in the number of expert trajectories and environment timesteps as demonstrated on MuJoCo benchmarks.

\label{sec:discussion}


\acknowledgments{\footnotesize
This paper is an extension on the course project of CMU 10-708 Probabilistic Graphical Models in Spring 2020, and we thank the course staff to provide the platform. We thank the anonymous reviewers for their useful comments. LL is supported by the National Science Foundation (DGE-1745016). BE is supported by the Fannie and John Hertz Foundation and the National Science Foundation (DGE-1745016).}


{
\footnotesize
\bibliography{references}  
}

\normalsize

\onecolumn
\appendix

\section{Derivation and Proof}
\label{app:deriv}
This section provides the derivation and proof for the main paper. Section \ref{app:gradient} and \ref{sec:f-div} provide the derivation of Theorem \ref{thm:fdiv}, and section \ref{sec:proof_robust} provides the details about section \ref{sec:robust}.

\subsection{Analytical Gradient of State Marginal Distribution}
\label{app:gradient}

In this subsection, we start by deriving a general result - gradient of state marginal distribution w.r.t. parameters of the reward function. We will use this gradient in the next subsection \ref{sec:f-div} where we derive the gradient of $f$-divergence objective.

Based on the notation introduced in section \ref{sec:prelim}, we start by writing the probability of trajectory $\tau=(s_0,s_1,\dots,s_T)$ of fixed horizon $T$ under the optimal MaxEnt trajectory distribution for $r_\theta(s)$~\cite{ziebart2010modeling}.

\begin{equation}
    \rho_{\theta}(\tau) \propto \rho_0(s_0) \prod_{t=0}^{T-1} p(s_{t+1}|s_t,a_t) e^{\sum_{t=1}^{T} r_{\theta}(s_t)/\alpha}
\end{equation}

Let $p(\tau) = \rho_0(s_0) \prod_{t=0}^{T-1} p(s_{t+1}|s_t,a_t)$, which is the probability of the trajectory under the dynamics of the environment.

Explicitly computing the normalizing factor, we can write the distribution over trajectories as follows:
\begin{equation}
    \rho_{\theta}(\tau) = \frac{p(\tau) e^{\sum_{t=1}^{T} r_{\theta}(s_t)/\alpha}}{\int p(\tau) e^{\sum_{t=1}^{T} r_{\theta}(s_t)/\alpha}d\tau}
\end{equation}

Let $\eta_{\tau}(s)$ denote the number of times a state occurs in a trajectory $\tau$. We now compute the marginal distribution of all states in the trajectory:
\be
    \rho_{\theta}(s) \propto \int p(\tau) e^{\sum_{t=1}^{T} r_{\theta}(s_t)/\alpha} \eta_{\tau}(s)d\tau
\ee
where
\be
   \eta_{\tau}(s)= \sum_{t=1}^T \mathbbm{1}(s_t = s)
\ee
is the empirical frequency of state $s$ in trajectory $\tau$ (omitting the starting state $s_0$ as the policy cannot control the initial state distribution).

The marginal distribution over states can now be written as:\\
\be
    \rho_{\theta}(s) \propto \int p(\tau) e^{\sum_{t=1}^{T} r_{\theta}(s_t)/\alpha}  \eta_{\tau}(s) d\tau
\ee

In the following derivation, we will use $s_t$ to denote states in trajectory $\tau$ and $s'_t$ to denote states from trajectory $\tau'$. Explicitly computing the normalizing factor, the marginal distribution can be written as follows:
\begin{equation}
\begin{split}
        \rho_{\theta}(s) &= \frac{\int p(\tau) e^{\sum_{t=1}^{T} r_{\theta}(s_t)/\alpha} \eta_{\tau}(s)d\tau}{\int \int p(\tau') e^{\sum_{t=1}^{T} r_{\theta}(s'_t)/\alpha} \eta_{\tau'}(s')d\tau' ds'}\\
    &= \frac{\int p(\tau) e^{\sum_{t=1}^{T} r_{\theta}(s_t)/\alpha} \eta_{\tau}(s)d\tau}{\int p(\tau') e^{\sum_{t=1}^{T}  r_{\theta}(s'_t)/\alpha} \int \eta_{\tau'}(s') ds' d\tau' }\\
    &= \frac{\int p(\tau) e^{\sum_{t=1}^{T} r_{\theta}(s_t)/\alpha} \eta_{\tau}(s)d\tau}{T \int p(\tau') e^{\sum_{t=1}^{T}  r_{\theta}(s'_t)/\alpha} d\tau' }
\end{split}
\end{equation}

In the second step we swap the order of integration in the denominator. The last line follows because only the $T$ states in $\tau$ satisfy $s \in \tau$. Finally, we define $f(s)$ and $Z$ to denote the numerator (dependent on $s$) and denominator (normalizing constant), to simplify notation in further calculations. 
\begin{equation}
    \begin{split}
        f(s) &= \int p(\tau) e^{\sum_{t=1}^{T} r_{\theta}(s_t)/\alpha} \eta_{\tau}(s)d\tau\\
    Z &= T \int p(\tau) e^{\sum_{t=1}^{T}  r_{\theta}(s_t)/\alpha} d\tau \\
    \rho_{\theta}(s) &= \frac{f(s)}{Z}
    \end{split}
\end{equation}
    
As an initial step, we compute the derivatives of $f(s)$ and $Z$ w.r.t reward function at some state $r_{\theta}(s^*)$.
\begin{align}
    \frac{df(s)}{dr_{\theta}(s^*)}  = \frac1\alpha\int p(\tau) e^{\sum_{t=1}^{T} r_{\theta}(s_t)/\alpha} \eta_{\tau}(s)\eta_{\tau}(s^*)d\tau \\
    \frac{dZ}{dr_{\theta}(s^*)} = \frac T\alpha \int p(\tau) e^{\sum_{t=1}^{T} r_{\theta}(s_t)/\alpha}  \eta_{\tau}(s^*)d\tau = \frac T\alpha f(s^*)
\end{align}

We can then apply the quotient rule to compute the derivative of policy marginal distribution w.r.t. the reward function. 
\begin{equation}
\begin{split}
      \frac{d\rho_{\theta}(s)}{dr_{\theta}(s^*)} &= \frac{Z\frac{df(s)}{dr_{\theta}(s^*)}-f(s)\frac{dZ}{dr_{\theta}(s^*)}}{Z^2}\\
     &=  \frac{\int p(\tau) e^{\sum_{t=1}^{T} r_{\theta}(s_t)/\alpha}\eta_{\tau}(s) \eta_{\tau}(s^*)d\tau}{\alpha Z} - \frac{f(s)}{Z}\frac{Tf(s^*)}{\alpha Z}\\
     &= \frac{\int p(\tau) e^{\sum_{t=1}^{T} r_{\theta}(s_t)/\alpha} \eta_{\tau}(s) \eta_{\tau}(s^*)d\tau}{\alpha Z} - \frac T \alpha \rho_{\theta}(s)\rho_{\theta}(s^{*})
     \end{split}
\end{equation}
  
Now we have all the tools needed to get the derivative of $\rho_{\theta}$ w.r.t. $\theta$ by the chain rule.
\begin{equation}
\label{eq:rho_grad}
\begin{split}
 \frac{d\rho_{\theta}(s)}{d\theta} &= \int \frac{d\rho_{\theta}(s)}{dr_{\theta}(s^*)} \frac{dr_{\theta}(s^*)}{d\theta} ds^*\\
    &= \frac1\alpha\int \left(\frac{\int p(\tau) e^{\sum_{t=1}^{T} r_{\theta}(s_t)/\alpha} \eta_{\tau}(s) \eta_{\tau}(s^*)d\tau}{Z} - T \rho_{\theta}(s)\rho_{\theta}(s^{*})\right) \frac{dr_{\theta}(s^*)}{d\theta} ds^*\\
    &= \frac{1}{\alpha Z} \int\int p(\tau) e^{\sum_{t=1}^{T} r_{\theta}(s_t)/\alpha} \eta_{\tau}(s)\eta_{\tau}(s^*) \frac{dr_{\theta}(s^*)}{d\theta} ds^* d\tau 
    - \frac T \alpha \rho_{\theta}(s)\int \rho_{\theta}(s^*)\frac{dr_{\theta}(s^*)}{d\theta}ds^*\\
    &= \frac{1}{\alpha Z} \int   p(\tau) e^{\sum_{t=1}^{T} r_{\theta}(s_t)/\alpha} \eta_{\tau}(s) \sum_{t=1}^{T} \frac{dr_{\theta}(s_t)}{d\theta}d\tau - \frac T \alpha\rho_{\theta}(s)\int \rho_{\theta}(s^*)\frac{dr_{\theta}(s^*)}{d\theta}ds^*
\end{split}
\end{equation}

\subsection{Analytical Gradient of $f$-divergence objective}
\label{sec:f-div}

$f$-divergence~\cite{ali1966general} is a family of divergence, which generalizes forward/reverse KL divergence. Formally, let $P$ and $Q$ be two probability distributions over a space $\Omega$, then for a convex and lower-semicontinuous function $f$ such that $f(1)=0$, the $f$-divergence of $P$ from $Q$ is defined as:

\begin{equation}
\fdiv{P}{Q}\defeq \int_\Omega f\left(\frac{dP}{dQ}\right)dQ
\end{equation}

Applied to state marginal matching between expert density $\rho_E(s)$ and agent density $\rho_\theta(s)$ over state space $\mathcal S$, the $f$-divergence objective is:

\begin{equation}
\min_\theta L_f(\theta) \defeq \fdiv{\rho_E}{\rho_\theta} = \int_{\mathcal S} f\left(\frac{\rho_E(s)}{\rho_\theta(s)}\right)\rho_\theta(s)ds
\end{equation}

Now we show the proof of \textbf{Theorem \ref{thm:fdiv}} on the gradient of $f$-divergence objective:

\begin{proof}
The gradient of the $f$-divergence objective can be derived by chain rule:

\begin{equation}
\begin{split}
\nabla_\theta L_f(\theta) &= \int \nabla_\theta \left(f\left(\frac{\rho_E(s)}{\rho_\theta(s)}\right)\rho_\theta(s)\right)d s\\
&= \int \left(f\left(\frac{\rho_E(s)}{\rho_\theta(s)}\right) - f'\left({\frac{\rho_E(s)}{\rho_\theta(s)}}\right) \frac{\rho_E(s)}{\rho_\theta(s)}\right)\frac{d \rho_\theta(s)}{d\theta} ds\\
&\triangleq  \int h_f\left(\frac{\rho_E(s)}{\rho_\theta(s)}\right)\frac{d \rho_\theta(s)}{d\theta} ds
\end{split}
\end{equation}
where we denote $h_f(u) \triangleq f(u) - f'(u)u.$ for convenience.\footnote{Note that if $f(u)$ is non-differentiable at some points, such as $f(u)=|u-1|/2$ at $u=1$ for Total Variation distance, we take one of its subderivatives.}

Substituting the gradient of state marginal distribution w.r.t $\theta$ in Eq. \ref{eq:rho_grad}, we have:
\begin{equation}
\footnotesize
\begin{split}
    &\nabla_\theta L_f(\theta) \\
    &=  \int h_f\left(\frac{\rho_E(s)}{\rho_\theta(s)}\right) \left(\frac{1}{\alpha Z} \int   p(\tau) e^{\sum_{t=1}^{T} r_{\theta}(s_t)/\alpha} \eta_{\tau}(s) \sum_{t=1}^{T} \frac{dr_{\theta}(s_t)}{d\theta}d\tau - \frac T \alpha\rho_{\theta}(s)\int \rho_{\theta}(s^*)\frac{dr_{\theta}(s^*)}{d\theta}ds^*\right)ds \\
    &= \frac{1}{\alpha Z} \int   p(\tau) e^{\sum_{t=1}^{T} r_{\theta}(s_t)/\alpha} \sum_{t=1}^{T}h_f\left(\frac{\rho_E(s_t)}{\rho_\theta(s_t)}\right) \sum_{t=1}^{T} \frac{dr_{\theta}(s_t)}{d\theta}d\tau \\
    &\quad - \frac T\alpha \int h_f\left(\frac{\rho_E(s)}{\rho_\theta(s)}\right) \rho_{\theta}(s) \left(\int\rho_{\theta}(s^*)\frac{dr_{\theta}(s^*)}{d\theta}ds^*\right) ds \\
    &= \frac{1}{\alpha T} \int \rho_{\theta}(\tau)\sum_{t=1}^{T}h_f\left(\frac{\rho_E(s_t)}{\rho_\theta(s_t)}\right) \sum_{t=1}^{T} \frac{dr_{\theta}(s_t)}{d\theta}d\tau \\
    &\quad - \frac T \alpha \left(\int h_f\left(\frac{\rho_E(s)}{\rho_\theta(s)}\right) \rho_{\theta}(s) ds\right)\left(\int \rho_{\theta}(s^*)\frac{dr_{\theta}(s^*)}{d\theta}ds^*\right) \\
     &= \frac{1}{\alpha T} \E{\tau \sim \rho_{\theta}(\tau)}{\sum_{t=1}^{T}h_f\left(\frac{\rho_E(s_t)}{\rho_\theta(s_t)}\right) \sum_{t=1}^{T}\frac{dr_{\theta}(s_t)}{d\theta}}  - \frac T\alpha \E{s \sim\rho_{\theta}(s)}{h_f\left(\frac{\rho_E(s)}{\rho_\theta(s)}\right)}\E{s \sim\rho_{\theta}(s)}{\frac{dr_{\theta}(s)}{d\theta}}
\end{split}
\end{equation}

To gain more intuition about this equation, we can convert all the expectations to be over the trajectories:
\begin{equation}
\footnotesize
\begin{split}
   &\nabla_\theta L_f(\theta) \\
   &= \frac{1}{\alpha T} \left(
   \E{\rho_{\theta}(\tau)}{\sum_{t=1}^{T}h_f\left(\frac{\rho_E(s_t)}{\rho_\theta(s_t)}\right) \sum_{t=1}^{T}\nabla_\theta r_{\theta}(s_t)}
   -
   \E{\rho_{\theta}(\tau)}{\sum_{t=1}^T h_f\left(\frac{\rho_E(s_t)}{\rho_\theta(s_t)}\right)}
   \E{\rho_{\theta}(\tau)}{\sum_{t=1}^T\nabla_\theta r_{\theta}(s_t)}
   \right)\\
   &= \frac{1}{\alpha T}\mathrm{cov}_{\tau\sim \rho_{\theta}(\tau)}\left(\sum_{t=1}^{T} h_f\left(\frac{\rho_E(s_t)}{\rho_{\theta}(s_t)}\right),\sum_{t=1}^{T} \nabla_\theta r_\theta(s_t)\right)
\end{split}
\end{equation}
Thus we have derived the analytic gradient of $f$-divergence for state-marginal matching as shown in Theorem \ref{thm:fdiv}.
\end{proof}

\subsection{Extension to Integral Probability Metrics in $f$-IRL}

Integral Probability Metrics (IPM)~\cite{muller1997integral} is another class of divergence based on dual norm, examples of which include Wasserstein distance~\cite{arjovsky2017wasserstein} and MMD~\cite{li2017mmd}. We can use Kantorovich-Rubinstein duality~\cite{villani2008optimal} to rewrite the IPM-based state marginal matching as:
\begin{equation}
\begin{split}
L_{\mathrm{B}}(\theta) = \|\rho_E(s) - \rho_\theta(s)\|_B \defeq \max_{D_\omega \in B} \E{\rho_E(s)}{D_\omega(s)} - \E{\rho_\theta(s)}{D_\omega(s)}
\end{split}
\end{equation}
where $B$ is a symmetric convex set of functions and $D_\omega$ is the critic function in \cite{arjovsky2017wasserstein}.

Then the analytical gradient of the objective $L_{\mathrm{B}}(\theta)$ can be derived to be:
\begin{equation}
\begin{split}
   \nabla_\theta L_{\mathrm{B}}(\theta) = -\frac{1}{\alpha T}\mathrm{cov}_{\tau\sim \rho_{\theta}(\tau)}\left(\sum_{t=1}^{T} D_\omega\left(s_t\right),\sum_{t=1}^{T} \nabla_\theta r_\theta(s_t)\right)
\end{split}
\end{equation}
where the derivation directly follows the proof of Theorem~\ref{thm:fdiv}.

\subsection{$f$-IRL Learns Disentangled Rewards w.r.t. Dynamics}
\label{sec:proof_robust}

We follow the derivation and definitions as given in~\citet{fu2017learning} to show that $f$-IRL learns disentangled rewards. We show the definitions and theorem here for completeness. For more information, please refer to~\citet{fu2017learning}. 

We first redefine the notion of ``disentangled rewards".

\begin{definition}[Disentangled rewards]
A reward function $r'(s,a,s')$ is (perfectly) disentangled with respect to ground truth reward $r_{\text{gt}}(s,a,s')$ and a set of dynamics $\mathcal{T}$ such that under all dynamics in $T~\in~\mathcal{T}$, the optimal policy is the same: $\pi^{*}_{r',T}(a|s) = \pi^{*}_{r_{\text{gt}},T}(a|s)$
\end{definition}

Disentangled rewards can be loosely understood as learning a reward function which will produce the same optimal policy as the ground truth reward for the environment, on any underlying dynamics.

To show how $f$-IRL recovers a disentangled reward function, we need go through the definition of "Decomposability condition"
 
\begin{definition}[Decomposability condition]
Two states $s_1$,$s_2$ are defined as "1-step linked" under a dyanamics or transition distribution $T(s'|a,s)$, if there exists a state that can reach $s_1$ and $s_2$ with positive probability in one timestep. Also, this relationship can transfer through transitivity: if $s_1$ and $s_2$ are linked, and $s_2$ and $s_3$ are linked then we can consider $s_1$ and $s_3$ to be linked.\\
A transition distribution $T$ satisfies the decomposibility condition if all states in the MDP are linked with all other states.
\end{definition} 
 
This condition is mild and can be satisfied by any of the environments used in our experiments. 
 
Theorem \ref{thm:disentangled_rewards} and \ref{thm:state_only_reward} formalize the claim that $f$-IRL recovers disentangled reward functions with respect to the dynamics. The notation $Q^{*}_{r,T}$ denotes the optimal Q function under reward function $r$ and dynamics $T$, and similarly $\pi^{*}_{r,T}$ is the optimal policy under reward function $r$ and dynamics $T$.

\begin{theorem}
\label{thm:disentangled_rewards}
Let $r_{gt}(s)$ be the expert reward, and $T$ be a dynamics satisfying the decomposability condition as defined in~\cite{fu2017learning}. Suppose $f$-IRL learns a reward $r_{\text{IRL}}$ such that it produces an optimal policy in $T$: $Q^*_{r_{\text{IRL}},T}(s,a)= Q^*_{r_{gt},T}(s,a)-f(s)$ ,where $f(s)$ is an arbitrary function of the state. Then we have:

$r_{\text{IRL}}(s)=r_{gt}(s)+C$ for some constant $C$, and thus $r_{\text{IRL}}(s)$ is robust to all dynamics.
\end{theorem}
\begin{proof}
Refer to Theorem 5.1 of AIRL~\cite{fu2017learning}.
\end{proof}

\begin{theorem}
\label{thm:state_only_reward}
If a reward function $r'(s,a,s')$ is disentangled with respect to all dynamics functions, then it must be state-only. 
\end{theorem}
\begin{proof}
Refer to Theorem 5.2 of AIRL~\cite{fu2017learning}.
\end{proof}

\section{What Objective is Optimized by Previous IL Algorithms?}

\label{sec:critique}
In this section, we discuss previous IL methods and analyze which objectives they may truly optimize. Our analysis shows that AIRL and GAN-GCL methods possibly optimize for a different objective than they claim, due to their usage of biased importance sampling weights.

\ifrss
\subsubsection{\textbf{MaxEntIRL \cite{ziebart2008maximum} and Deep MaxEntIRL \cite{wulfmeier2015maximum}}}
\else
\subsection{MaxEntIRL \cite{ziebart2008maximum}, Deep MaxEntIRL \cite{wulfmeier2015maximum}, GCL~\cite{finn2016guided}}
\fi

Classical IRL methods~\cite{russell1998learning, ng2000algorithms} obtain a policy by learning a reward function from sampled trajectories of an expert policy. MaxEntIRL~\citep{ziebart2008maximum} learns a stationary reward by maximizing likelihood on expert trajectories, i.e., it minimizes forward KL divergence in trajectory space under the maximum entropy RL framework. A trajectory is a temporal collection of state-action pairs, and this makes the trajectory distribution different from state-action marginal or state marginal distribution. Each objective - minimizing divergence in trajectory space $\tau$, in state-action marginal space $(s,a)$ and state marginal $s$ are different IL methods in their own sense.

MaxEntIRL derives a surrogate objective w.r.t. reward parameter as the difference in cumulative rewards of the trajectories between the expert and the soft-optimal policy under current reward function. To train the soft-optimal policy, it requires running MaxEnt RL in an inner loop after every reward update. This algorithm has been successfully applied for predicting behaviors of taxi drivers with a linear parameterization of reward. \citet{wulfmeier2015maximum} shows that MaxEntIRL reward function can be parameterized as deep neural networks as well.

Guided cost learning (GCL)~\cite{finn2016guided} is one of the first methods to train rewards using neural network directly through experiences from real robots. They achieve this result by leveraging guided policy search for policy optimization, employing importance sampling to correct for distribution shift when the policy has not converged, and using novel regularizations in reward network. GCL optimizes for the same objective as MaxEntIRL and Deep MaxEntIRL. To summarize these three works, we have the following observation:

\begin{obs}
\label{obs:maxentirl}
MaxEntIRL, Deep MaxEntIRL, GCL all optimize for the forward KL divergence in trajectory space, i.e. $\kl{\rho_{E}(\tau)}{\rho_{\theta}(\tau)}$.
\end{obs}

\ifrss
\subsubsection{\textbf{GAN-GCL~\cite{finn2016connection}, AIRL~\cite{fu2017learning}}}
\else
\subsection{ GAN-GCL~\cite{finn2016connection}, AIRL~\cite{fu2017learning}, EAIRL \cite{qureshi2018adversarial}}
\fi

\citet{finn2016connection} shows that GCL is equivalent to training GANs with a special structure in the discriminator (GAN-GCL). Note that this result uses an approximation in importance sampling, and hence the gradient estimator is \textit{biased}. \citet{fu2017learning} shows that GAN-GCL does not perform well in practice since its discriminator models density ratio over trajectories which leads to high variance. They propose an algorithm AIRL in which the discriminator estimates the density ratio of state-action marginal, and shows that AIRL empirically performs better than GAN-GCL. AIRL also uses approximate importance sampling in its derivation, and therefore its gradient is also \textit{biased}. GAN-GCL and AIRL claim to be able to recover a reward function due to the special structure in the discriminator. EAIRL \cite{qureshi2018adversarial} uses empowerment regularization on policy objective based on AIRL.

All the above algorithm intend to optimize for same objective as MaxEntIRL. However, there is an approximation involved in the procedure and let us analyze what that is, by going through the derivation for \textit{equivalence of AIRL to MaxEntIRL} as shown in \citet{fu2017learning} (Appendix A of that paper).

The authors start from writing down the objective for MaxEntIRL: $\text{max}_{\theta}\, L_{\mathrm{MaxEntIRL}}(\theta) = \E{\tau \sim D}{\log p_{\theta}(\tau)}$, where $D$ is the collection of expert demonstrations, and reward function is parameterized by $\theta$.

When the trajectory distribution is induced by the soft-optimal policy under reward $r_\theta$, it can be parameterized as $p_{\theta}(\tau) \propto p(s_0) \prod_{t=0}^{T-1}p(s_{t+1}|s_t,a_t)e^{r_{\theta}(s_t,a_t)}$, then its gradient is derived as follows:

\begin{equation}
    \begin{split}
        \frac{d}{d\theta}L_{\mathrm{MaxEntIRL}}(\theta) & =  \E{D}{\frac{d}{d\theta}r_{\theta}(s_t,a_t)} - \frac{d}{d\theta}\log(Z_\theta)\\
        &= \E{D}{\sum_{t=1}^T \frac{d}{d\theta}r_{\theta}(s_t,a_t)} - \E{p_{\theta}}{\sum_{t=1}^T \frac{d}{d\theta}r_{\theta}(s_t,a_t)}\\
        &= \sum_{t=1}^T \E{D}{\frac{d}{d\theta}r_{\theta}(s_t,a_t)} - \E{p_{\theta,t}}{\frac{d}{d\theta}r_{\theta}(s_t,a_t)}
    \end{split}
\end{equation}
where $Z_\theta$ is the normalizing factor of $p_\theta(\tau)$, and $p_{\theta,t}(s_t,a_t) = \int_{s_{t'!=t},a_{t'!=t}}p_{\theta}(\tau)$ denote the state action marginal at time $t$.

As it is difficult to draw samples from $p_{\theta}$, the authors instead train a separate importance sampling distribution $\mu(\tau)$. For the choice of distribution they follow \cite{finn2016guided} and use a mixture policy $\mu(a|s) = 0.5 \pi(a|s) + 0.5 \hat{q}(a|s)$ where $\hat{q}(a|s)$ is the rough density estimate trained on the demonstrations. This is justified as reducing the variance of the importance sampling distribution. Thus the new gradient becomes:
\begin{equation}
    \frac{d}{d\theta}L_{\mathrm{MaxEntIRL}}(\theta) = \sum_{t=1}^T \E{D}{\frac{d}{d\theta}r_{\theta}(s_t,a_t)} - \E{\mu_t}{\frac{p_{\theta,t}(s_t,a_t)}{\mu_t(s_t,a_t)}\frac{d}{d\theta}r_{\theta}(s_t,a_t)}
    \label{eq:maxentirl}
\end{equation}
We emphasize here $\hat{q}(a|s)$ is the density estimate trained on the demonstrations.

They additionally aim to adapt the importance sampling distribution to reduce variance by minimizing $\kl{\pi(\tau)}{p_{\theta}(\tau)}$, and this KL objective can be simplified to the following MaxEnt RL objective:

\begin{equation}
    \text{max}_{\pi}\E{\pi}{\sum_{t=1}^T r_{\theta}(s_t,a_t) - \log \pi(a_t|s_t)}
\end{equation}

This ends the derivation of gradient of MaxEntIRL. Now, AIRL tries to show that the gradient of AIRL matches the gradient for MaxEntIRL objective shown above, i.e. $\frac{d}{d\theta}L_{\mathrm{MaxEntIRL}}(\theta) = \frac{d}{d\theta}L_{\mathrm{AIRL}}(\theta)$, then AIRL is equivalent to MaxEntIRL to a constant, i.e. $L_{\mathrm{MaxEntIRL}}(\theta) = L_{\mathrm{AIRL}}(\theta) +C$.

In AIRL, the cost learning objective is replaced by training a discriminator of the following form:
\begin{equation}
    D_{\theta}(s,a) = \frac{e^{f_{\theta}(s,a)}}{e^{f_{\theta}(s,a)}+\pi(a|s)}
\end{equation}

The objective of the discriminator is to maximize the cross-entropy between the expert demonstrations and the generated samples:

\begin{equation}
    \begin{split}
        \max_\theta L_{\mathrm{AIRL}}(\theta) & = \sum_{t=1}^T \E{D}{\log D_{\theta}(s_t,a_t)} + \E{\pi_t}{\log(1-D_{\theta}(s_t,a_t))}\\
        & = \sum_{t=1}^T \E{D}{\log \frac{e^{f_\theta(s_t,a_t)}}{e^{f_\theta(s_t,a_t)}+\pi(a_t|s_t)}} + \E{\pi_t}{\log \frac{\pi(a_t|s_t)}{\pi(a_t|s_t)+e^{f_\theta(s_t,a_t)}}}\\
        & = \sum_{t=1}^T \E{D}{f_\theta(s_t,a_t)} + \E{\pi_t}{\log \pi(a_t|s_t)} -2 \E{\mu_t}{\log(\pi(a_t|s_t))+e^{f_\theta(s_t,a_t)}}
    \end{split}
\end{equation}
where $\mu_t$ is the mixture of state-action marginal from expert demonstrations and from state-action marginal induced by current policy $\pi$ at time $t$.

In AIRL, the policy $\pi$ is optimized with the following reward:
\begin{equation}
\begin{split}
    \hat{r}(s,a) &= \log(D_{\theta}(s,a)) - \log(1-D_{\theta}(s,a))\\
    &= f_\theta(s,a) - \log \pi(a|s)
\end{split}
\end{equation}

Taking the derivative with respect to $\theta$,
\begin{equation}
    \frac{d}{d\theta}L_{\mathrm{AIRL}}(\theta) = \sum_{t=1}^T \E{D}{\frac{d}{d\theta}f_\theta(s_t,a_t)} - \E{\mu_t}{\frac{e^{f_\theta(s_t,a_t)}}{ (e^{f_\theta(s_t,a_t)}+ \pi(a_t|s_t))/2} \frac{d}{d\theta}f_\theta(s_t,a_t)}
\end{equation}
The authors multiply state marginal $\pi(s_t) = \int_a \pi_t(s_t,a_t)$ to the fraction term in the second expectation, and denote that $\hat{p}_{\theta,t}(s_t,a_t)\triangleq e^{f_\theta(s_t,a_t)}\pi(s_t)$ and $\hat{\mu}_t(s_t,a_t)\triangleq (e^{f_\theta(s_t,a_t)}+\pi(a_t|s_t))\pi(s_t)/2$.

Thus the gradient of the discriminator becomes:
\begin{equation}
\label{eq:airl_is}
    \frac{d}{d\theta}L_{\mathrm{AIRL}}(\theta) = \sum_{t=1}^T \E{D}{\frac{d}{d\theta}f_{\theta}(s_t,a_t)} - \E{\mu_t}{\frac{\hat{p}_{\theta,t}(s_t,a_t)}{\hat \mu_t(s_t,a_t)}\frac{d}{d\theta}f_{\theta}(s_t,a_t)}
\end{equation} 

\subsubsection{AIRL is Not Equivalent to MaxEntIRL}

The issues occurs when AIRL tried to match Eq. \ref{eq:airl_is} with Eq. \ref{eq:maxentirl}, i.e. $\frac{d}{d\theta}L_{\mathrm{MaxEntIRL}}(\theta) \question \frac{d}{d\theta}L_{\mathrm{AIRL}}(\theta)$ with same reward parameterization $f_\theta=r_\theta$. 

If they are equivalent, then we have the importance weights equality:
\begin{equation}
\boxed{\label{eq:issue}
    \hat{p}_{\theta,t}(s_t,a_t) = p_{\theta,t}(s_t,a_t), \hat{\mu}_t(s_t,a_t)={\mu}_t(s_t,a_t)
}
\end{equation}

Then given the definitions, we have:

\begin{equation}
\begin{split}
    \hat{p}_{\theta,t}(s_t,a_t) &\triangleq e^{f_\theta(s_t,a_t)}\pi(s_t)=\pi^*_\theta(s_t)\pi^*_\theta(a_t|s_t)\triangleq p_{\theta,t}(s,a)\\
    \hat{\mu}_t(s_t,a_t) &\triangleq  (e^{f_\theta(s_t,a_t)}+\pi(a_t|s_t))\pi(s_t)/2 = (\pi(a_t|s_t)+\hat q(a_t|s_t))\pi(s_t)/2 \triangleq \mu_t(s_t,a_t)
\end{split}
\end{equation}
where $\pi_\theta^*$ is soft-optimal policy under reward $r_\theta=f_\theta$ (assumption), thus $\log \pi_\theta^*(a_t|s_t)=f_\theta(s_t,a_t)$. Then equivalently,
\begin{equation}
\begin{split}
\label{eq:all_eq}
\boxed{
    e^{f_\theta(s_t,a_t)} = \hat q(a_t|s_t) = \pi^*_\theta(a_t|s_t) = \pi(a_t|s_t)
}
\end{split}
\end{equation}

Unfortunately these equivalences hold only at the global optimum of the algorithm, when the policy $\pi$ reaches the expert policy $\pi_E\approx \hat q$ and the discriminator is also optimal. This issue also applies to GAN-GCL and EAIRL. Therefore, we have the following  conclusion:

\begin{obs}
\label{obs:airl}
GAN-GCL, AIRL, EAIRL are \textbf{not} equivalent to MaxEntIRL, i.e. \textbf{not} minimizing forward KL in trajectory space and possibly optimizing a biased objective. 
\end{obs}

\subsubsection{AIRL is Not Optimizing Reverse KL in State-Action Marginal}

$f$-MAX \cite{ghasemipour2019divergence} (refer to their Appendix C) states that AIRL is equivalent to $f$-MAX with $f = -\log u$. This would imply that AIRL minimizes reverse KL in state-action marginal space.

However, there are some differences in AIRL algorithm and the $f$-MAX algorithm with reverse KL divergence. $f$-MAX\citep{ghasemipour2019divergence} considers a vanilla discriminator. This is different than the original AIRL~\cite{fu2017learning}, which uses a specially parameterized discriminator. To highlight this difference we refer to $f$-MAX with $f = -log u$ (called AIRL in their paper) as $f$-MAX-RKL in this paper, since it aims to minimize reverse-KL between  state-action marginal. We see below that using f-MAX method with special discriminator(instead of vanilla) might not correspond to reverse KL minimization in state-action marginal which shows that AIRL does not truly minimize reverse KL divergence.

To show the equivalence of AIRL to reverse KL matching objective, \citet{ghasemipour2019divergence} considers that the AIRL discriminator can be trained till convergence. With the special discriminator of AIRL, at convergence the following equality holds:
\begin{equation}
    \frac{e^{f_\theta(s,a)}}{e^{f_{\theta}(s,a)}+\pi(a|s)} \equiv \frac{\rho_{E}(s,a)}{\rho_{E}(s,a) + \rho_\theta(s,a)} \,\,\text{(at convergence)}
\end{equation}
but if this is true then $f_{\theta}(s,a)$ can no longer be interpreted as the stationary reward function as it is a function of current policy:
\begin{equation}
    f_\theta(s,a) = \frac{\rho_{E}(s,a)}{\rho_\theta(s,a)} \pi(a|s)
\end{equation}

\begin{obs}
\label{obs:airl_rkl}
AIRL is \textbf{not} optimizing reverse KL in state-action marginal space.
\end{obs}

\ifrss
\subsubsection{\textbf{FAIRL \cite{ghasemipour2019divergence}}}
\else
\subsection{GAIL~\cite{ho2016generative},
FAIRL, $f$-MAX-RKL \cite{ghasemipour2019divergence}}
\fi

Generative Adversarial Imitation Learning (GAIL)~\cite{ho2016generative} addresses the issue of running RL in an inner step by adversarial training~\cite{goodfellow2014generative}. A discriminator learns to differentiate over state-action marginal and a policy learns to maximize the rewards acquired from the discriminator in an alternating fashion. It can be further shown that the GAIL is minimizing the Jensen-Shannon divergence over state-action marginal given optimal discriminator. 

Recently the idea of minimizing the divergence between expert and policy's marginal distribution is further comprehensively studied and summarized in~\citet{ke2019imitation} and ~\citet{ghasemipour2019divergence}, where the authors show that any $f$-divergence can be minimized for imitation through $f$-GAN framework \cite{nowozin2016f}. $f$-MAX proposes several instantiations of $f$-divergence: forward KL for $f$-MAX-FKL (FAIRL), reverse KL for $f$-MAX-RKL, and Jensen-Shannon for original GAIL. Their objectives are summarized as below, where $\theta$ is policy parameter, $f^*$ is the convex conjugate of $f$ and $T_\omega$ is the discriminator.

\begin{equation}
\begin{split}
\min_\theta \fdiv{\rho_E(s,a)}{\rho_\theta(s,a)} = \min_\theta \max_\omega \E{(s,a)\sim \rho_E(s,a)}{T_\omega(s,a)} - \E{(s,a)\sim \rho_\theta(s,a)}{f^*(T_\omega(s,a))} 
\end{split}
\end{equation}

These adversarial IRL methods cannot recover a reward function because they do minimax optimization with discriminator in the inner-loop (when optimal, the discriminator predicts $1\over 2$ everywhere), and have poorer convergence guarantees opposed to using an analytical gradient.

\begin{obs}
\label{obs:gail}
GAIL, FAIRL, $f$-MAX-RKL are optimizing JS, forward KL, and reverse KL in state-action marginal space, respectively without recovering a reward function.
\end{obs}

\subsection{SMM~\cite{lee2019efficient}}
\citet{lee2019efficient} presents state marginal matching (SMM) for efficient exploration by minimizing reverse KL between expert and policy's state marginals (Eq \ref{eq:smm_objective}). However, their method cannot recover the stationary reward function because it uses fictitious play between policy $\pi_\theta$ and variational density $q$, and requires storing a historical average of policies and densities over previous iterations.
\begin{equation}
    \max_\theta -\kl{\rho_{\theta}(s)}{\rho_{E}(s)} = \max_\theta \E{\rho_\theta(s)}{\log \frac{\rho_E(s)}{\rho_\theta(s)}} = \max_\theta \min_q \E{\rho_\theta(s)}{\log \frac{\rho_E(s)}{q(s)}} 
 \label{eq:smm_objective}
\end{equation}

\subsection{Summary of IL/IRL Methods: Two Classes of Bilevel Optimization}

Now we generalize the related works including our method into \textit{reward-dependent} and \textit{policy-dependent} classes from the viewpoint of optimization objective.

For the \textbf{reward-dependent} (IRL) methods such as MaxEntIRL, AIRL, and our method, the objective of reward/discriminator $r_\theta$ and policy $\pi_\phi$ can be viewed as a bilevel optimization:
\begin{equation}
\begin{split}
&\min_{\theta,\phi} L(r_\theta, \pi_\phi)\\
&\mathrm{s.t.}\,\phi \in \arg \max_\phi g(r_\theta, \pi_\phi) 
\end{split}
\end{equation}
where $L(\cdot, \cdot)$ is the joint loss function of reward and policy, and $g(r, \cdot)$ is the objective of policy given reward $r$.
Thus the optimal policy is dependent on current reward, and training on the final reward does produce optimal policy, i.e. recovering the reward.

For the \textbf{policy-dependent} (IL) method such as $f$-MAX, GAIL, and SMM, the objective of reward/discriminator $r_\theta$ and policy $\pi_\phi$ can be viewed as:
\begin{equation}
\begin{split}
\max_{\phi} \min_\theta L(r_\theta, \pi_\phi)
\end{split}
\end{equation}
This is a special case of bilevel optimization, minimax game. The optimal reward is dependent on current policy as the inner objective is on reward, thus it is non-stationary and cannot guarantee to recover the reward.

\section{Implementation Details}
\label{appendix:implementation}

\subsection{Matching the Specified Expert State Density on Reacher (Sec \ref{sec:matching-expert-density})}
\label{sec:reacher_detail}
\textbf{Environment:} The OpenAI gym \texttt{Reacher-v2} environment \cite{brockman2016openai} has a robotic arm with 2 DOF on a 2D arena. The state space is 8-dimensional: sine and cosine of both joint angles, and the position and velocity of the arm fingertip in x and y direction. The action controls the torques for both joints. The lengths of two bodies are $r_1=0.1,r_2=0.11$, thus the trace space of the fingertip is an annulus with $R=r_1+r_2=0.21$ and $r=r_2-r_1=0.01$. Since $r$ is very small, it can be approximated as a disc with radius $R=0.21$. The time horizon is $T=30$. We remove the object in original reacher environment as we only focus on the fingertip trajectories.

\textbf{Expert State Density:} The domain is x-y coordinate of fingertip position. We experiment with the following expert densities:
\begin{itemize}
  \item \textit{Single Gaussian:} $\mu=(-R,0)=(-0.21,0), \sigma=0.05$.
  \item \textit{Mixture of two equally-weighted Gaussians:} $\mu_1=(-R/\sqrt{2},-R/\sqrt{2}), \mu_2=(-R/\sqrt{2},R/\sqrt{2}), \sigma_1=\sigma_2=0.05$
\end{itemize}

\textbf{Training Details:} We use SAC as the underlying RL algorithm for all compared methods. 
The policy network is a tanh squashed Gaussian, where the mean and std is parameterized by a (64, 64) ReLU MLP with two output heads. The Q-network is a (64, 64) ReLU MLP. We use Adam to optmize both the policy and the Q-network with a learning rate of 0.003. The temperature parameter $\alpha$ is fixed to be 1. The replay buffer has a size of 12000, and we use a batch size of 256.

For $f$-IRL and MaxEntIRL, the reward function is a (64, 64) ReLU MLP. We clamp the output of the network to be within the range [-10, 10]. We also use Adam to optimize the reward network with a learning rate of 0.001.

For other baselines including AIRL, $f$-MAX-RKL, GAIL, we refer to the $f$-MAX~\cite{ghasemipour2019divergence} authors' official implementation\footnote{\url{https://github.com/KamyarGh/rl_swiss}}. 
We use the default discriminator architecture as in~\cite{ghasemipour2019divergence}. In detail, first the input is
linearly embedded into a 128-dim vector. This hidden state then passes through 6 Resnet blocks of
128-dimensions; the residual path uses batch normalization and tanh activation. The last hidden
state is then linearly embedded into a single-dimensional output, which is the logits of the discriminator. The logit is clipped to be within the range $[-10,10]$. The discriminator is optimized using Adam with a learning rate of 0.0003 and a batch size of 128.

At each epoch, for all methods, we train SAC for 10 episodes using the current reward/discriminator. We warm-start SAC policy and critic networks from networks trained at previous iteration. We do not empty the replay buffer, and leverage data collected in earlier iterations for training SAC. We found this to be effective empirically, while saving lots of computation time for the bilevel optimization.

For $f$-IRL and MaxEntIRL, we update the reward for 2 gradient steps in each iteration. For AIRL, $f$-MAX-RKL and GAIL, the discriminator takes 60 gradient steps per epoch. We train all methods for 800 epochs.

$f$-IRL and MaxEntIRL require an estimation of the agent state density. We use kernel density estimation to fit the agent's density, using epanechnikov kernel with a bandwidth of 0.2 for pointmass, and a bandwidth of 0.02 for Reacher. At each epoch, we sample 1000 trajectories (30000 states) from the trained SAC to fit the kernel density model.

\textbf{Baselines:} Since we assume only access to expert density instead of expert trajectories in traditional IL framework, we use \textit{importance sampling} for the expert term in the objectives of baselines.
\begin{itemize}
    \item \textit{For MaxEntIRL}: Given the reward is only dependent on state, its reward gradient can be transformed into covariance in state marginal space using importance sampling from agent states: 
    \begin{equation}
    \label{eq:maxentirl}
    \begin{split}
    \nabla_\theta {L}_{\mathrm{MaxEntIRL}}(\theta)
    &=\frac{1}{\alpha}\sum_{t=1}^T\left( \E{s_t \sim \rho_{E,t}}{\nabla  r_\theta(s_t)} - \E{s_t \sim \rho_{\theta,t}}{\nabla  r_\theta(s_t)}\right)\\
    &= \frac{T}{\alpha} \left( \E{s \sim \rho_E}{ \nabla r_\theta(s)} - \E{s \sim \rho_\theta}{ \nabla r_\theta(s)}\right)\\
    &= \frac{T}{\alpha} \left( \E{s \sim \rho_\theta}{\frac{\rho_E(s)}{\hat \rho_\theta(s)} \nabla r_\theta(s)} - \E{s \sim \rho_\theta}{ \nabla r_\theta(s)}\right)
    \end{split}
    \end{equation}
    where $\rho_{t}(s)$ is state marginal at timestamp $t$, and $\rho(s)=\sum_{t=1}^T \rho_t(s)/T$ is state marginal averaged over all timestamps, and we fit a density model to the agent distribution as $\hat{\rho}_{\theta}$.
    
    \item \textit{For GAIL, AIRL, $f$-MAX-RKL}: 
    Original discriminator needs to be trained using expert samples, thus we use the same density model as described above, and then use importance sampling to compute the discriminator objective:
    \begin{equation}
            \max_D L(D) = \E{s\sim \rho_{\theta}}{\frac{\rho_E(s)}{\hat{\rho}_{\theta}(s)} \log D(s)} +\E{s\sim\rho_\theta}{\log(1 - D(s))}
    \end{equation}
\end{itemize}

\textbf{Evaluation:} For the approximation of both forward and reverse KL divergence, we use non-parametric Kozachenko-Leonenko estimator~\cite{kozachenko1987sample, kraskov2004estimating} with lower error \cite{singh2016analysis} compared to  plug-in estimators using density models. Suggested by \cite{ver2000non}\footnote{\url{https://github.com/gregversteeg/NPEET}}, we choose $k=3$ in $k$-nearest neighbor for Kozachenko-Leonenko estimator. Thus for each evaluation, we need to collect agent state samples and expert samples for computing the estimators.

In our experiments, before training we sample $M=10000$ expert samples and keep the valid ones within observation space. For agent, we collect $1000$ trajectories of $N=1000*T=30000$ state samples. Then we use these two batches of samples to estimate KL divergence for every epoch during training.

\subsection{Inverse Reinforcement Learning Benchmarks (Sec \ref{sec:irl-experiments})}
\label{sec:mujoco_detail}
\textbf{Environment:} We use the \texttt{Hopper-v2}, \texttt{Ant-v2}, \texttt{HalfCheetah-v2}, \texttt{Walker2d-v2} environments from OpenAI Gym.

\textbf{Expert Samples:} We use SAC to train expert policies for each environment. SAC uses the same policy and critic networks, and the learning rate as section \ref{sec:reacher_detail}. We train using a batch size of 100, a replay buffer of size 1 million, and set the temperature parameter $\alpha$ to be 0.2. The policy is trained for 1 million timesteps on Hopper, and for 3 million timesteps on the other environments. All algorithms are tested on 1, 4, and 16 trajectories collected from the expert stochastic policy.

\textbf{Training Details:} We train $f$-IRL, Behavior Cloning (BC), MaxEntIRL, AIRL, and $f$-MAX-RKL to imitate the expert using the provided expert trajectories.

We train $f$-IRL using Algorithm \ref{algo:full}. Since we have access to expert samples, we use the practical modification described in section \ref{sec:trick} for training $f$-IRL, where we feed a mixture of 10 agent and 10 expert trajectories (resampled with replacement from provided expert trajectories) into the reward objective.

SAC uses the same hyperparameters used for training expert policies.  Similar to the previous section, we warm-start the SAC policy and critic using trained networks from previous iterations, and train them for 10 episodes. 
At each iteration, we update the reward parameters once using Adam optimizer. For the reward network of $f$-IRL and MaxEntIRL, we use the same reward structure as section \ref{sec:reacher_detail} with the learning rate of $0.0001$, and $\ell_2$ weight decay of $0.001$. We take one gradient step for the reward update.

MaxEntIRL is trained in the standard manner, where the expert samples are used for estimating reward gradient.

For Behavior cloning, we use the expert state-action pairs to learn a stochastic policy that maximizes the likelihood on expert data. The policy network is same as the one used in SAC for training expert policies. 

For $f$-MAX-RKL and AIRL, we tuned the hyperparameters based on the code provided by $f$-MAX that is used for \textit{state-action} marginal matching in Mujoco benchmarks. For $f$-MAX-RKL, we fix SAC temperature $\alpha=0.2$, and tuned reward scale $c$ and gradient penalty coefficient $\lambda$ suggested by the authors, and found that $c=0.2,\lambda=4.0$ worked for $\{$Ant, Hopper, Walker2d$\}$ \textit{with} the normalization in each dimension of states and with a replay buffer of size 200000. However, for HalfCheetah, we found it only worked with $c=2.0,\lambda=2.0$ \textit{without} normalization in states and with a replay buffer of size 20000. For the other hyperparameters and training schedule, we keep them same as $f$-MAX original code: e.g. the discriminator is parameterized as a two-layer MLP of hidden size 128 with tanh activation and the output clipped within [-10,10]; the discriminator and policy are alternatively trained once for 100 iterations per 1000 environment timesteps.

For AIRL, we re-implement a version that uses SAC as the underlying RL algorithm for a fair comparison, whereas the original paper uses TRPO. Both the reward and the value model are parameterized as a two-layer MLP of hidden size 256 and use ReLU as the activation function. For SAC training, we tune the learning rates and replay buffer sizes for different environments, but find it cannot work on all environments other than HalfCheetah even after tremendous tuning. For reward and value model training, we tune the learning rate for different environments. These hyper-parameters are summarized in table~\ref{tab:airl_parameters}. We set $\alpha=1$ in SAC for all environments. For every 1000 environment steps, we alternatively train the policy and the reward/value model once, using a batch size of 100 and 256.

\begin{table}[h]
    \centering
    \vspace{1mm}
    \begin{tabular}{ccccc}
    \hline
    Hyper-parameter & Ant & Hopper & Walker & HalfCheetah \\
    \hline
    SAC learning rate  & $3e-4$ &  $1e-5$  & $1e-5$  & $3e-4$ \\
    SAC replay buffer size & 1000000 & 1000000 & 1000000 & 10000 \\
    Reward/Value model learning rate & $1e-4$ & $1e-5$ & $1e-5$  & $1e-4$ \\
    \hline
    \end{tabular}
    \vspace{1mm}
    \caption{AIRL IRL benchmarks task-specific hyper-parameters.}
    \label{tab:airl_parameters}
\end{table}

\textbf{Evaluation:} We compare the trained policies by $f$-IRL, BC, MaxEntIRL, AIRL, and $f$-MAX-RKL by computing their returns according to the ground truth return on each environment. We report the mean of their performance across 3 seeds.

For the IRL methods, $f$-IRL, MaxEntIRL, and AIRL, we also evaluate the learned reward functions. We train SAC on the learned rewards, and evaluate the performance of learned policies according to ground-truth rewards.

\subsection{Reward Prior for Downstream Hard-exploration Tasks (Sec \ref{sec:exp-downstream-tasks}.1)}
\label{sec:prior_detail}

\textbf{Environment:} The pointmass environment has 2D square state space with range $[0,6]^2$, and 2D actions that control the delta movement of the agent in each dimension. The agent starts from the bottom left corner at coordinate $(0,0)$. 

\textbf{Task Details:} We designed a hard-to-explore task for the pointmass. The grid size is $6 \times 6$, the agent is always born at $[0, 0]$, and the goal is to reach the region $[5.95, 6] \times [5.95, 6]$. The time horizon is $T=30$. The agent only receives a reward of 1 if it reaches the goal region. To make the task more difficult, we add two distraction goals: one is at $[5.95, 6] \times [0, 0.05]$, and the other at $[0, 0.05] \times [5.95, 6]$. The agent receives a reward of $0.1$ if it reaches one of these distraction goals. Vanilla SAC always converges to reaching one of the distraction goals instead of the real goal. 

\textbf{Training Details: }
We use SAC as the RL algorithm. We train SAC for 270 episodes, with a batch size of 256, a learning rate of 0.003, and a replay buffer size of 12000. To encourage the exploration of SAC, we use a random policy for the first 100 episodes.

\subsection{Reward Transfer across Changing Dynamics (Sec \ref{sec:exp-downstream-tasks}.2)}
\label{sec:transfer_detail}

\begin{figure}[h]
    \centering
    \includegraphics[width=1.\textwidth]{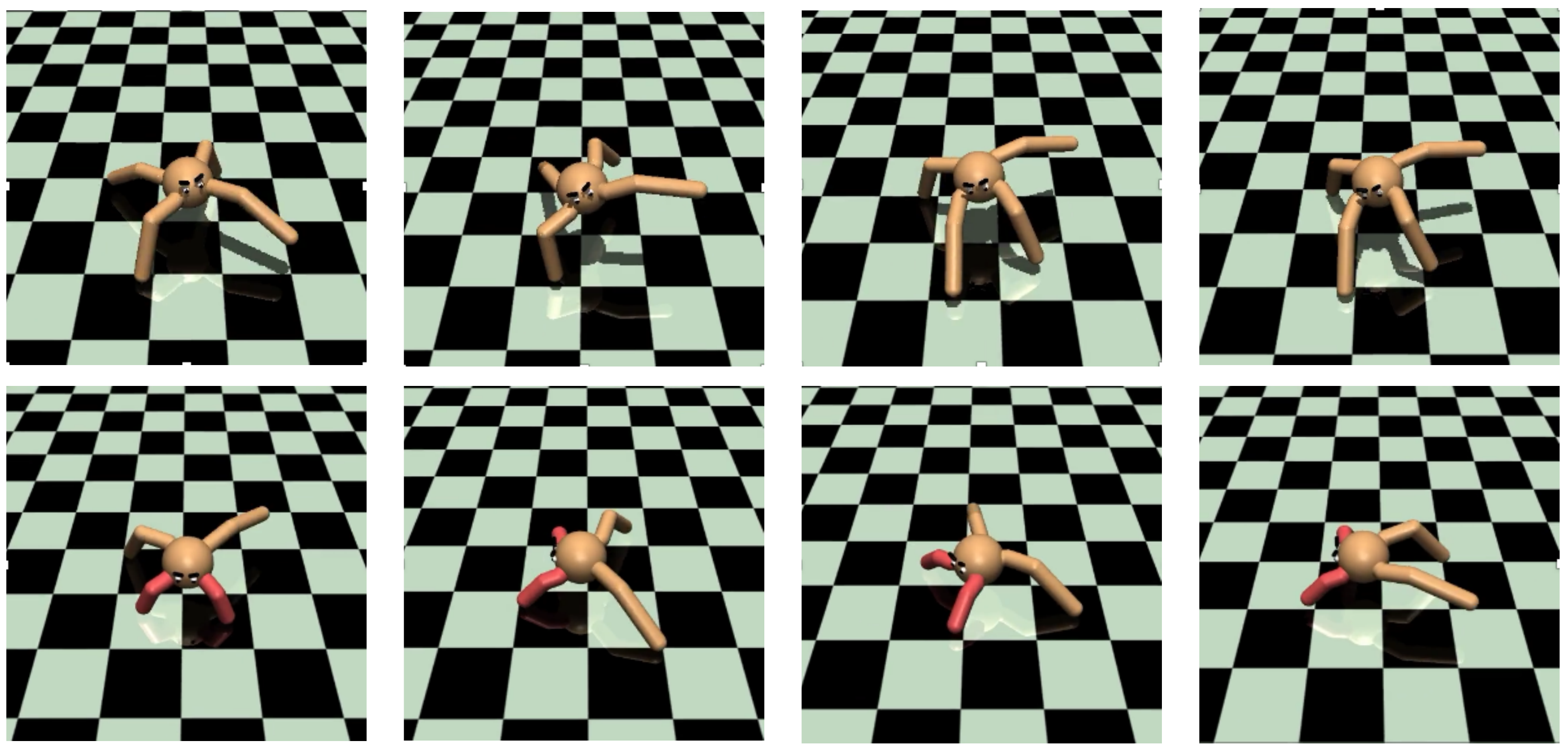}
    \caption{\textbf{Top row}: A healthy Ant executing a forward walk. \textbf{Bottom row}: A successful transfer of walking behavior to disabled Ant with 2 legs active. The disabled Ant learns to use the two disabled legs as support and crawl forward, executing a very different gait than previously seen in healthy Ant.}
    \label{fig:ant_transfer}
\end{figure}

\textbf{Environment:} In this experiment, we use Mujoco to simulate a healthy Ant, and a disabled Ant with two broken legs (Figure \ref{fig:ant_transfer}). We use the code provided by \citet{fu2017learning}. Note that this Ant environment is a slightly modified version of the \texttt{Ant-v2} available in OpenAI gym. 

\textbf{Expert Samples:} We use SAC to obtain a forward-running policy for the Ant. We use the same network structure and training parameters as section \ref{sec:mujoco_detail} for training this policy. We use 16 trajectories from this policy as expert demonstrations for the task. 

\textbf{Training Details:} We train $f$-IRL and MaxEntIRL using the same network structure and training parameters as section \ref{sec:mujoco_detail}. We also run AIRL, but couldn't match the performance reported in \citet{fu2017learning}.

\textbf{Evaluation:} We evaluate $f$-IRL and MaxEntIRL by training a policy on their learned rewards using SAC. We report the return of this policy on the disabled Ant environment according to the ground-truth reward for forward-running task. Note that we directly report results for policy transfer using GAIL, and AIRL from \citet{fu2017learning}.

\section{Additional Experiment Results}
\label{app:result}

\subsection{Inverse RL Benchmark Unnormalized Performance}
\label{sec:additional_mujoco}

In this section, we report the unnormalized return of the agent stochastic policy for ease of comparison to expert in Table \ref{tab:mujoco_unnormalized}. We analyze situations when we are provided with 1,4 and 16 expert trajectories respectively. For IL/IRL methods, all the results are averaged across three random seeds to show the mean and standard deviation in the last $10\%$ training iterations.

Note that for the row of ``Expert return", we compute the mean and std among the expert trajectories (by stochastic policy) we collected, so for one expert trajectory, it does not have std. Moreover, since we pick the best expert trajectories for training IL/IRL algorithms, the std of ``Expert return" is often lower than that of IL/IRL.

\begin{table}[h]
    \vspace{2mm}
    \centering

    \begin{tabular}{c|ccc}
        \toprule
         Method & \multicolumn{3}{c}{Hopper} \\
        \midrule
         \# Expert traj
           & 1 & 4 & 16\\
        Expert return &
    3570.87 & 3585.59 $\pm$ 12.39 & 3496.62 $\pm$ 10.13 \\
       \midrule
        BC 
           & 17.39 $\pm$ 5.99 & 468.49 $\pm$ 83.94 & 553.56 $\pm$ 46.70\\
        MaxEntIRL  
           & 3309.72 $\pm$ 171.28 & 3300.81 $\pm$ 229.84 & \textbf{3298.50} $\pm$ 255.35\\
        $f$-MAX-RKL  
           & \textbf{3349.62} $\pm$ 68.89 & 3326.83 $\pm$ 85.42 & 3165.51 $\pm$ 102.83 \\
        AIRL  & 49.12 $\pm$ 2.58 & 49.33 $\pm$ 3.93 & 48.63 $\pm$ 5.88 \\
           \midrule
        FKL ($f$-IRL)  
           & 3329.94 $\pm$ 152.33 & 3243.83 $\pm$ 312.44 & 3260.35 $\pm$ 175.58 \\
        RKL ($f$-IRL)  
           & 3276.55 $\pm$ 221.27 & 3303.44 $\pm$ 286.26 & 3250.74 $\pm$ 161.89
           \\
        JS ($f$-IRL)  
           & 3282.37 $\pm$ 202.30 & \textbf{3351.99} $\pm$ 172.70 & 3269.49 $\pm$ 160.99
           \\
        \bottomrule
    \end{tabular}
    \vspace{2mm}

    \begin{tabular}{c|ccc}
        \toprule
         Method & \multicolumn{3}{c}{Walker2d} \\
        \midrule
         \# Expert traj
           & 1 & 4 & 16\\
        Expert return &
    5468.36 & 5337.85 $\pm$ 92.46 & 5368.01 $\pm$ 78.99 \\
       \midrule

        BC 
           & -2.03 $\pm$ 1.05 & 303.24 $\pm$ 6.95 & 431.60 $\pm$ 63.68\\
         MaxEntIRL  
           & 4823.82 $\pm$ 512.58 & 4697.11  $\pm$ 852.19 & \textbf{4884.30}  $\pm$ 467.16\\
        $f$-MAX-RKL  
           & 2683.11 $\pm$ 128.14 & 2628.10 $\pm$ 548.93 & 2498.78 $\pm$ 824.26 \\
        AIRL  & 9.8 $\pm$ 1.82 & 9.24 $\pm$ 2.28 & 8.45 $\pm$ 1.56 \\
           \midrule
        FKL ($f$-IRL)  
           & \textbf{4927.02} $\pm$ 615.34 & 4809.80 $\pm$ 750.05 & 4851.81 $\pm$ 547.12 \\
        RKL ($f$-IRL)  
           & 4847.12 $\pm$ 806.61 & 4806.72 $\pm$ 433.02 & 4578.39 $\pm$ 564.17
           \\
        JS ($f$-IRL)  
           & 4888.09 $\pm$ 664.86 & \textbf{4935.42} $\pm$ 384.15 & 4725.78 $\pm$ 613.45
           \\
        \bottomrule
    \end{tabular}

    \vspace{2mm}
    \begin{tabular}{c|ccc}
        \toprule
         Method & \multicolumn{3}{c}{HalfCheetah} \\
        \midrule
         \# Expert traj
           & 1 & 4 & 16\\
        Expert return &
    12258.71 & 11944.45 $\pm$ 985.08 & 12406.29 $\pm$ 614.02 \\
       \midrule

        BC 
           & -367.56 $\pm$ 23.57 & 209.59 $\pm$ 178.71 & 287.05 $\pm$ 109.32\\
        MaxEntIRL  
           & \textbf{11637.41} $\pm$ 438.16 & 11685.92 $\pm$ 478.95 & 11228.32 $\pm$ 1752.32\\
        $f$-MAX-RKL  
           & 8688.37 $\pm$ 633.58 & 4920.66 $\pm$ 2996.15 & 8108.81 $\pm$ 1186.77 \\
        AIRL  & 2366.84 $\pm$ 175.51 & 2343.17 $\pm$ 103.51 & 2267.68 $\pm$ 83.59 \\
           \midrule
        FKL ($f$-IRL)  
           & 11556.23 $\pm$ 539.83 & 11556.51 $\pm$ 673.13 & 11642.72 $\pm$ 629.29 \\
        RKL ($f$-IRL)  
           & 11612.46 $\pm$ 703.25 & 11644.19 $\pm$ 488.79 & \textbf{11899.50} $\pm$ 605.43
           \\
        JS ($f$-IRL)  
           & 11413.47 $\pm$ 1227.89 & \textbf{11686.09} $\pm$ 748.30 & 11711.77 $\pm$ 1091.74
           \\
        \bottomrule
    \end{tabular}

    \vspace{2mm}
    \begin{tabular}{c|ccc}
        \toprule
         Method & \multicolumn{3}{c}{Ant} \\
        \midrule
         \# Expert traj
           & 1 & 4 & 16\\
        Expert return &
    5926.18 & 5859.09 $\pm$ 88.72 & 5928.87 $\pm$ 136.44 \\
       \midrule

        BC 
           & -113.60 $\pm$ 12.86 & 1321.69 $\pm$ 172.93 & 2799.34 $\pm$ 298.93 \\
        MaxEntIRL  
           & 3179.23 $\pm$ 2720.63 & 4171.28 $\pm$ 1911.67 & 4784.78 $\pm$ 482.01\\
        $f$-MAX-RKL  
           & 3585.03  $\pm$ 255.91 & 3810.56 $\pm$ 252.57 & 3653.53 $\pm$ 403.73 \\
        AIRL  & -54.7 $\pm$ 28.5 & -14.15 $\pm$ 31.65 & -49.68 $\pm$ 41.32 \\
           \midrule
        FKL ($f$-IRL)  
           & \textbf{4859.86} $\pm$ 302.94 & \textbf{4861.91} $\pm$  452.38 & \textbf{4971.11} $\pm$ 286.81 \\
        RKL ($f$-IRL)  
           & 3707.32 $\pm$ 2277.74 & 4814.58 $\pm$ 376.13 & 4813.80 $\pm$ 361.93
           \\
        JS ($f$-IRL)  
           & 4590.11 $\pm$ 1091.22 & 4745.11 $\pm$ 348.97 & 4342.39 $\pm$ 1296.93
           \\
        \bottomrule
    \end{tabular}
    \vspace{2mm}
    \caption{Benchmark of Mujoco Environment, from top to bottom, Hopper-v2, Walker2d-v2, HalfCheetah-v2, Ant-v2.}
    \label{tab:mujoco_unnormalized}
\end{table} 

\begin{figure}[h]
\vspace{2mm}
    \centering
    \begin{tabular}{ccccc}
    \multicolumn{3}{c}{\includegraphics[width=.5\columnwidth]{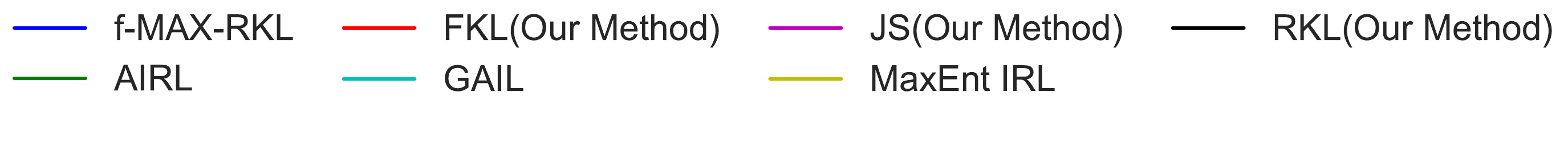}} & FKL & RKL  \\
    \multicolumn{3}{c}{\includegraphics[width=.5\columnwidth]{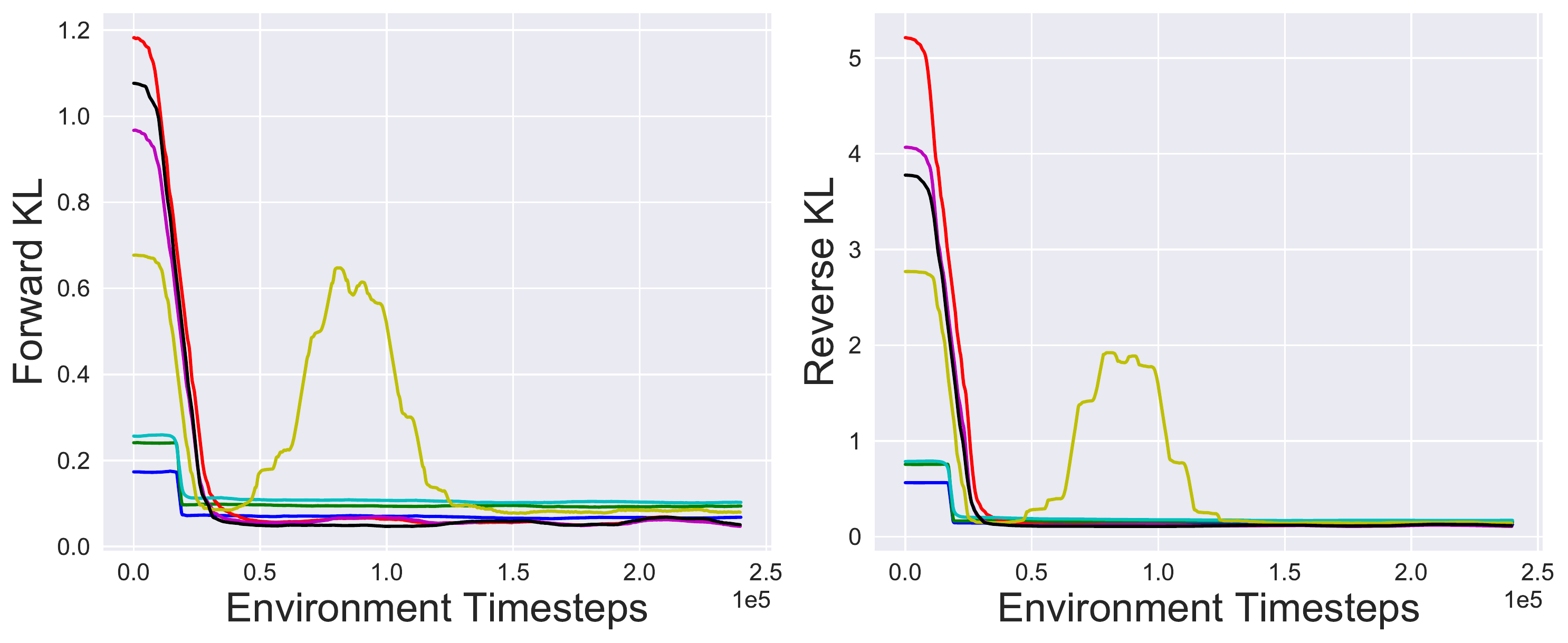}} & \includegraphics[width=0.22\columnwidth]{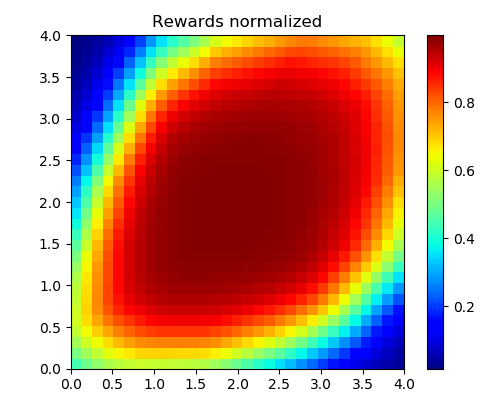} & \includegraphics[width=0.22\columnwidth]{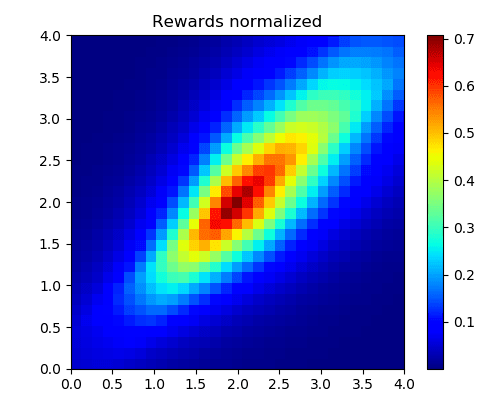} \\
    \end{tabular}
 \caption{Left: Forward and Reverse KL curves in pointmass environment for Gaussian target density for all the methods during training. Smoothed in a window of 120 evaluations. Right: Learned rewards by $f$-IRL by optimizing for Forward KL (left) and Reverse KL (right) objective in pointmass goal-reaching task.}
    \label{fig:didactic}
\end{figure}
\subsection{Additional Result of Reward Transfer across Changing Dynamics}

\begin{wraptable}{R}{0.65\linewidth}
    \centering
    \footnotesize
    \begin{tabular}{cccc|c}
        \toprule
       Policy Transfer  & AIRL &
        MaxEntIRL &$f$-IRL & Ground-truth\\
        using GAIL &
         & & & Reward\\
        \midrule
         -29.9 & 130.3 & 145.5& \textbf{232.5} & 315.5\\ 
        \bottomrule
    \end{tabular}
    \caption{Returns obtained after transferring the policy/reward on modified Ant environment using different IL methods. In this case, we report the performance of best seed with a maximum of 50 expert trajectories.}
    \label{tab:policy-transfer_new}
    \vspace{1mm}
\end{wraptable}

In section \ref{sec:exp-downstream-tasks} (``Reward transfer across changing dynamics") we show the result of the setting with 32 expert trajectories provided. We follow AIRL paper~\cite{fu2017learning} setting for this experiment, but the number of experiment trajectories used in their experiment is \textit{unknown}. We use a maximum of 50 expert trajectories and show the best seed performance in Table \ref{tab:policy-transfer_new}. Note that this table has same values as Table~\ref{tab:policy-transfer} except for our method. We see that with more expert trajectories $f$-IRL is able to outperform baselines with a large margin. The disabled Ant agent is able to learn a behavior to walk while taking support from both of its disabled legs.

\subsection{Matching the Specified Expert State Density on PointMass}

We also conducted the experiment described in section \ref{sec:reacher_detail} on the pointmass environment similar to that in section \ref{sec:prior_detail}. This environment has size $[4,4]$, and the target density is a unimodal Gaussian with $\mu=(2,2), \sigma=0.5$ for goal-reaching task.

This experiment is didactic in purpose. 
In the left of Figure \ref{fig:didactic}, we observe that all methods converge (MaxEntIRL is slightly unstable) and are able to reduce the FKL and RKL to near zero.

In the right of Figure \ref{fig:didactic}, we observe that rewards learned by $f$-IRL using Forward KL and Reverse KL divergence objective demonstrate the expected \textit{mode-covering} and \textit{mode-seeking} behavior, respectively.

\clearpage

\section*{Generalization to Real Robots}
\label{app:real_robot}

\subsection*{$f$-IRL achieves better simulation results than prior work which has been demonstrated to work on real robots.} \citet{finn2016guided} showed that GCL, which is a variant of MaxEnt IRL with importance sampling correction, can work well on real robots.  As shown in Figure 3 and Table 3, $f$-IRL achieves similar or higher final policy performances, and has slightly better sample-efficiency (in terms of environment interaction steps) compared with MaxEntIRL, in all four Mujoco simulated locomotion tasks. Therefore, we believe that $f$-IRL should be no harder to run on physical robots.

\subsection*{$f$-IRL is more sample-efficient than prior IL/IRL methods.}
It is critical for IL/IRL methods to have a good sample-efficiency, in terms of both number of expert samples and number of environment interaction steps, when applied to real robots. In practise, it is difficult to collect large number of expert trajectories for certain tasks. Moreover, larger environment interactions steps slow down the real-robot experiment (as robots move much slower in real world than in simulator) and increase the chances of tear/damage to the robot. As shown in Figure 3 and Table 3, $f$-IRL has better sample efficiency in both expert samples and environment steps when compared with prior work. Therefore, we believe it would be much easier to apply $f$-IRL to real robots compared with these prior works.

\subsection*{$f$-IRL requires less hyperparameter-tuning for each task.}

Tuning hyperparameters in the real world is quite arduous, so a method that works for a wide range of hyperparameter settings would be ideal.
As shown in appendix \ref{appendix:implementation} on implementation details, $f$-IRL uses the same set of hyperparameters for all four Mujoco benchmarks (one exception is that we increase the reward network size for Ant due to its large state space), and we did not spend much time in tuning the hyperparameters. 

On the contrary, in our re-implementation, we find prior methods such as AIRL~\cite{fu2017learning} and the AIL methods~\cite{ho2016generative, ghasemipour2019divergence} extremely sensitive to hyperparameters. For example, we find that $f$-MAX works only with observation normalization in Ant, Hopper and Walker and only without observation normalization in HalfCheetah. $f$-MAX also requires tedious hyperparameter grid search on reward scale and gradient penalty, as mentioned in their appendix~\cite{ghasemipour2019divergence} and verified in our re-implementation. We believe that the ease of tuning of $f$-IRL will make it much easier to be applied to real robots compared with these prior works.

\end{document}